\documentclass{article}

\usepackage{iclr2027_conference,times}

\usepackage[T1]{fontenc}
\usepackage[utf8]{inputenc}
\usepackage{microtype}

\usepackage{amsmath,amssymb,amsthm,mathtools}
\usepackage{bm}

\newcommand{\secref}[1]{\S\ref{#1}}

\newcommand{\appref}[1]{Appendix~\ref{#1}}

\usepackage{booktabs}
\usepackage{graphicx}
\usepackage{subcaption}
\usepackage{xcolor}
\usepackage{listings}
\usepackage{threeparttable}
\usepackage{multirow}
\usepackage[scaled=0.95]{inconsolata}
\usepackage{caption}
\usepackage{placeins}

\lstdefinestyle{introcode}{
    language=Python,
    basicstyle=\ttfamily\scriptsize,
    columns=fullflexible,
    keepspaces=true,
    showstringspaces=false,
    numbers=left,
    numberstyle=\tiny,
    numbersep=6pt,
    frame=single,
    framerule=0.35pt,
    moredelim=**[is][\bfseries]{@}{@}
}

\usepackage{wrapfig}

\usepackage[normalem]{ulem}

\usepackage{titletoc}
\usepackage[hidelinks]{hyperref}
\usepackage[nameinlink,noabbrev]{cleveref}

\crefname{figure}{Fig.}{Figs.}
\Crefname{figure}{Fig.}{Figs.}

\crefname{table}{Tab.}{Tabs.}
\Crefname{table}{Tab.}{Tabs.}

\crefname{section}{\S}{\S\S}
\Crefname{section}{\S}{\S\S}

\crefname{equation}{eq.}{eqs.}
\Crefname{equation}{Eq.}{Eqs.}

\crefname{proposition}{Prop.}{Props.}
\Crefname{proposition}{Prop.}{Props.}

\crefname{theorem}{Thm.}{Thms.}
\Crefname{theorem}{Thm.}{Thms.}

\crefname{corollary}{Cor.}{Cors.}
\Crefname{corollary}{Cor.}{Cors.}

\usepackage{enumitem}

\newcommand{\E}{\mathbb{E}}
\newcommand{\Var}{\operatorname{Var}}

\newcommand{\Pois}{\operatorname{Pois}}

\newtheorem{theorem}{Theorem}[section]
\newtheorem{proposition}[theorem]{Proposition}
\newtheorem{lemma}[theorem]{Lemma}
\newtheorem{corollary}[theorem]{Corollary}
\theoremstyle{definition}

\theoremstyle{remark}

\title{Extending Pathwise Gradients to Discrete Random Variables via Finite-Order Relaxation}
\author{Donghan He \\
\texttt{donghanhe@gmail.com}
\And
Luhuan Wu \\
Department of Applied Mathematics and Statistics,\\
Johns Hopkins University\\
\texttt{lwu86@jh.edu}
}

\iclrfinalcopy

\begin{document}

\maketitle
\lhead{Preprint}

\begin{abstract}
Pathwise gradients are preferred for continuous random variables because they are unbiased, low variance, and work with a single sample.
For discrete variables, 
pathwise gradients cannot generally be exact for every differentiable function.
We propose a general framework for constructing
\textbf{Fini}te-order exact pathwise gradient estimators
(Fini) for discrete variables, which are unbiased for all polynomials up to degree $D$.  
Among all admissible solutions, our estimator is the unique least-norm solution and has minimum weight variance. 
For common discrete distributions such as the Poisson, our construction yields native, closed-form  estimators tailored to each distribution, 
implementable in a few lines of code.
They preserve exact discrete forward samples without continuous relaxation, temperature tuning, or augmented representations. 
At matched polynomial exactness, they avoid the excess weight variance of non-native representations.  
We also bound bias for smooth downstream functions, showing how accuracy beyond polynomials depends on the first unmatched derivative and the sampling distribution.
In experiments our low-order estimators match or outperform tuned baselines across linear, nonlinear and hierarchical latent-variable models, while out-speeding competitors in every runtime benchmark.
\end{abstract}

\section{introduction}
For Monte Carlo (MC) gradient estimation in continuous random variables, the pathwise gradient estimator is the recommended default for its unbiasedness, typically low variance, and implementation simplicity \citep{mcge_survey}. 
Pathwise gradients measure sensitivity among sample paths, combining how a realized sample changes with the distribution's parameters with how the downstream function changes with the sample.
Thus it requires differentiability of both the sample path and the downstream function.
However, discrete variables, commonly used in neural sequence models and deep generative models \citep{def,pvae,second_gsm_poglm}, do not generally admit a pathwise estimator that is unbiased for all differentiable functions. 
This leaves a gap for MC gradient estimation in discrete distributions. 

Existing gradient estimators for discrete variables navigate a bias--variance--computation tradeoff. 
Score-function estimators are unbiased but often have high variance; reducing it may require more samples or carefully designed control variates, adding computational or implementation cost  \citep{bbvi, mcge_survey}.
Measure-valued derivatives (MVD) and discrete General One-sample gradients can reduce variance through additional downstream evaluations, which become expensive in high-dimensional models \citep{mvd,mvd_vi,go}.
Biased alternatives instead approximate the backward pass, trading exactness for inexpensive gradients.
Straight-through (ST) estimators preserve the hard sample but substitute a surrogate backward derivative, introducing bias \citep{ste}.
For categorical variables, the Gumbel--Softmax trick \citep{gsm,concrete} instead replaces the hard sample with a differentiable continuous approximation controlled by temperature, allowing gradients through the relaxed sample. 
Lowering the temperature improves bias but can increase gradient variance, 
making it an additional hyperparameter to tune.
Beyond categorical variables, researchers use the Gumbel--Softmax trick to approximate gradients for discrete variable such as Poisson \citep{second_gsm_poglm}; 
however, these methods need to approximate one-dimensional variables with $K$-dimensional categorical variables, and inherit the temperature parameter, both adding computational overhead. 

Recently, \cite{reinmax} 
interpret categorical ST estimators as a first-order numerical method and propose a second-order ST method, ReinMax. 
ReinMax is unbiased for degree-2 polynomials, whereas first-order ST is unbiased for linear functions.
This improvement enables ReinMax to outmatch both biased and unbiased methods consistently on various benchmarks. 
Nonetheless, applying ReinMax to other discrete variables again requires approximation, which can lead to statistical and computational overheads.
Can we adapt pathwise gradients to discrete variables beyond categoricals
by relaxing the unbiasedness requirement in a principled way? 

\begin{figure}[!t]
\centering
\begin{minipage}[t]{0.25\linewidth}
    \vspace{0pt}
    \centering
    \includegraphics[width=0.99\linewidth]{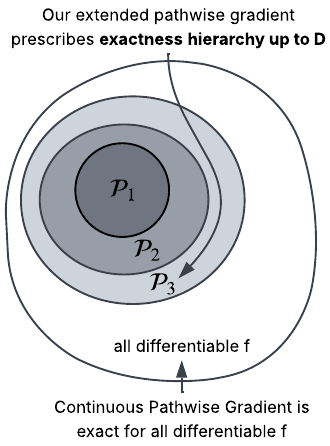}
\end{minipage}
\hfill
\begin{minipage}[t]{0.33\linewidth}
    \vspace{0pt}
    \centering
    \includegraphics[width=0.99\linewidth]{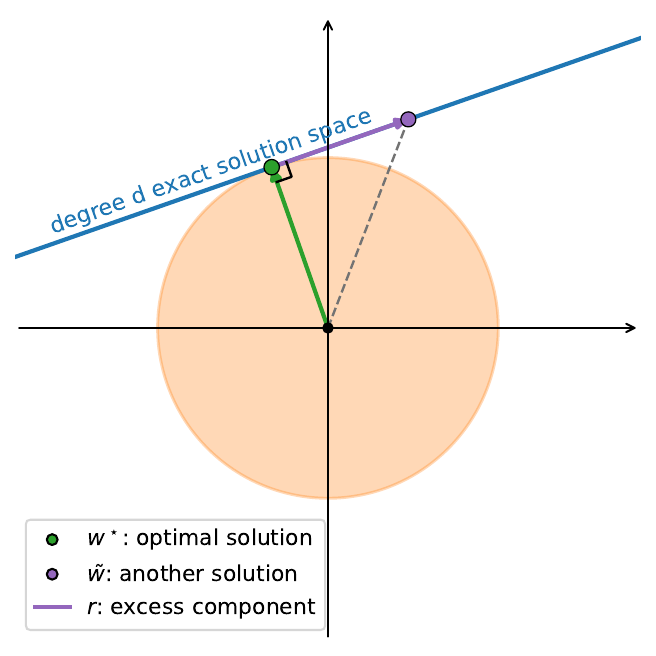}
\end{minipage}
\hfill
\begin{minipage}[t]{0.32\linewidth}
    \vspace{0pt}
    \captionsetup[lstlisting]{font=footnotesize,skip=2pt}
    \begin{lstlisting}[
        style=introcode,
        basicstyle=\ttfamily\fontsize{5.6}{6.2}\selectfont,
        numberstyle=\fontsize{4.8}{5.2}\selectfont,
        numbersep=3pt,
        xleftmargin=0.1em,
        xrightmargin=0.1em,
        caption={Poisson Fini-2 and Fini-3 in torch.},
        label={lst:intro-poisson},
        captionpos=t
    ]
...
rate = eta.exp()
z_hard = Poisson(rate).sample()

# order 2:
@w = 0.5 * (z_hard + rate)@
# order 3:
# w = (5*rate/12
#   + (8*rate+1)*z_hard/(12*rate)
#   - z_hard**2/(12*rate))

@z = z_hard + w.detach() * (eta - eta.detach())@
...    \end{lstlisting}
\end{minipage}
\vspace{0.1em}
\caption{
\textbf{Fini}te-order pathwise gradients. \textbf{Left}: pathwise gradients are unbiased for all differentiable functions - this is too rigid for discrete variables, so we relax the unbiasedness requirement to low-order polynomials; \textbf{Center}: generally, there are many solutions to the finite-order unbiasedness requirement; we use the least-norm principle to select the order-D unbiased solution, avoiding excess weight variance;
\textbf{Right}: as an example, Fini-2 and 3 for Poisson reduce to a few lines of code. 
}
\label{fig:intro-construction}
\end{figure}

We introduce a general recipe, Fini, that applies to a range of discrete variables:
As shown in \Cref{fig:intro-construction}, we require the pathwise estimator to be unbiased for low-degree polynomials rather than for every differentiable function; this weaker requirement admits many solutions, so we select the least-norm solution.
This leads to several mathematical, statistical, and computational benefits, 
shown below.  


\paragraph{Contributions:}
\begin{itemize}[leftmargin=*, itemsep=1pt, topsep=2pt]

\item
\textbf{Fini: Finite-order exact discrete pathwise gradients:}
We introduce Fini, a general recipe that extends pathwise gradients to discrete random variables by requiring the gradient to be unbiased for polynomials up to degree $D$ and selecting the least-norm solution.
(\secref{sec:method}).

\item
\textbf{Generality across discrete variables:}
For common discrete variables, the estimator is solved once at design time. The estimator is simple to implement (see the Poisson example in \Cref{lst:intro-poisson}) and evaluation is fast and independent of support size
(\secref{sec:method}, \appref{app:estimators},
\secref{sec:two-parameter-st}).

\item
\textbf{Bias and variance analysis:}
Our solution uniquely minimizes weight variance within the exactness class; we derive a non-asymptotic bias bound beyond it.
We also characterize excess variance introduced by categorical and compound representations (\secref{sec:theory}).

\item
\textbf{Better empirical performance:} 
Across neural sequence models, VAEs, and hierarchical topic models, Fini matches or outperforms tuned specialized and generalist baselines 
(\secref{sec:experiments}).

\end{itemize}

\section{Background}
\label{sec:background}

\subsection{Monte Carlo Gradient Estimation and Score Function Gradients}
\label{sec:mc-gradient-background}

For a probability distribution $p_{\theta}(Z)$, we write the associated expectation $\E_{p_{\theta}}$ as $\E_{\theta}$ for simplicity. 
For an objective $\E_{\theta}[f(Z)]$,
gradient-based learning requires calculating $\partial_\theta \E_{\theta}[f(Z)]$, often from MC samples. 
This problem emerges in supervised,
unsupervised, and reinforcement learning. 

One way to estimate $\partial_\theta \E_{\theta}[f(Z)]$ is through the score-function gradient, also known as REINFORCE \citep{reinforce,bbvi}. They use the identity
$
\partial_\theta\E_{\theta}[f(Z)]
=
\E_{p_\theta}\!\left[f(Z)s_\theta(Z)\right]
$, with $s_\theta(Z)=\partial_\theta\log p_\theta(Z)$.
Computing the integral analytically or numerically is generally intractable in high dimensions;
MC gradient estimation instead draws $N$ samples $z_n \sim p_\theta$ and forms unbiased estimators of the quantity using $\sum_{n=1}^N f(z_n) s_{\theta}(z_n)/N$.
The score-function estimator applies to both continuous and discrete variables and is unbiased under suitable regularity conditions.
The assumptions are quite general, but the estimator can have high variance 
\citep{mcge_survey}.

\subsection{Pathwise Gradients}
\label{sec:pathwise-background}



When $f$ is differentiable, pathwise gradients instead use derivative
information and integrate directly with ordinary backpropagation.
For a continuous family, a pathwise derivative weight $w_\theta$ satisfies
\begin{equation}
\partial_\theta\E_\theta[f(Z)]
=
\E_\theta\!\left[w_\theta(Z)f'(Z)\right]
\qquad
\text{for every differentiable }f.
\label{eq:continuous-pathwise-identity}
\end{equation}
Explicit reparameterization writes out the sampling path
$Z=g(\epsilon, \theta)$ for parameter-free noise $\epsilon$ and
differentiates the path so $w_\theta = \partial_{\theta} g(\epsilon, \theta)$ 
\citep{vae,stochastic_backprop}.
Consider a Gaussian $\mathcal{N}(\mu, \sigma)$, one valid path is $g(\epsilon, \theta=(\mu, \sigma)) = \mu+\sigma \epsilon$ where $\epsilon \sim \mathcal{N}(0, 1)$;
then one can use $f'(z) \partial_{\theta}g(\epsilon, \theta)$ as a valid one-sample estimator for $\partial_\theta L$.
For notation simplicity, we note $w_\theta$ as $w$ going forward; it generally depends on the distribution parameters.


The pathwise estimator is the recommended default for continuous MC gradient estimation because it is unbiased, typically low variance, and works well with one sample.
For common discrete variables, one can show no such $w$ can satisfy \Cref{eq:continuous-pathwise-identity} for all differentiable $f$. 
See proof in \Cref{app:bump-proof} for details.
This limitation motivates researchers to use estimators that introduce bias, high variance, or additional computations.
\section{Fini: \textbf{Fini}te-Order Pathwise Gradients}
\label{sec:method}
To extend pathwise gradients to discrete distributions, we relax the unbiasedness 
requirement in \Cref{eq:continuous-pathwise-identity}.
For a candidate pathwise weight $w$, instead of enforcing the unbiased pathwise identity for any differentiable $f$, we relax this requirement to unbiasedness with respect to polynomials up to degree $D$, denoted as $\mathcal{P}_D$.  
Since expectation is linear and the monomial basis spans the vector space of polynomials, unbiasedness over $\mathcal P_{D}$ breaks down to unbiasedness against $\{x, x^2,\cdots, x^{D}\}$.
Thus, with the pathwise form $w_D(Z)f'(Z)$ and a prescribed order $D$, the relaxed requirement becomes
\begin{equation}
\E_\theta[w_D(Z)f'(Z)]
=
\partial_\theta\E_\theta[f(Z)]
\quad
\text{for every }f\in \{Z^d\}_{d=1}^D
\label{eq:finite-order-pathwise}
\end{equation}
Given we want to find a solution $w^\star$ in the functional space, we need to be able to describe the space first, so we equip the space with the inner-product $\langle u,v\rangle_\theta=\E_\theta[u(Z)v(Z)]$. 
We have now arrived at a Hilbert space  $L^2(p_{\theta})$ given the base law $p_{\theta}$, and we see that for low orders, constraints in \Cref{eq:finite-order-pathwise} generally leave $w$ underdetermined.
We choose the least-norm rule to resolve the ambiguity, which leads to clean derivations and computationally cheap solutions as shown later in this section, and other desirable statistical properties shown in \Cref{sec:theory}. 
Assuming suitable integrability and differentiability, the Hilbert space optimization problem can be written as follows:

\begin{equation}
\begin{aligned}
&
\operatorname*{arg\,min}_{w_D\in L^2(p_\theta)}
\|w_D\|_2
\\
\text{s.t.}\qquad
\langle w_D,Z^{d-1}\rangle_\theta
&=
\frac1d\partial_\theta\E_\theta[Z^d],
\qquad
d=1,\ldots,D.
\end{aligned}
\label{eq:min-norm-moment-problem}
\end{equation}

The key to solving this problem is to invoke the Riesz representation theorem, and note that the monomial basis up to $D-1$ are precisely the Riesz representers to our problem.
Therefore, the solution is a linear combination of monomials $v_D(Z) = [1,..., Z^{D-1}]$,
which means the resulting weight has degree at most $D-1$.
Given the solution $w_D^\star$ must be a linear combination of monomials, each of the LHS in our $D$ constraints becomes a gram matrix $G_D=\E_\theta[v_D(Z)v_D(Z)^\top]$
multiplied by the linear combination's coefficients.
As long as the gram matrix is non-singular, we can safely invert it.
Collecting the RHS of constraints in problem (\ref{eq:min-norm-moment-problem}) as $b_D$, we see that the solution is: 
\begin{equation}
\boxed{
w_D^\star(z)
=
v_D(z)^\top G_D^{-1}b_D.
}
\label{eq:gram-solution}
\end{equation}

We refer to solution $w_D^\star$ as the Fini-$D$ estimator. 
This is the minimum-norm/least-norm solution to underspecified problems \citep{boyd2004, trefethen1997numerical}.
The recipe is general: we compute moments up to $2D-2$ for the gram matrix, as well as the first $D$ moments' derivatives for $b_D$, then solve the $D \times D$ linear system symbolically once to obtain the closed-form solution in \Cref{eq:gram-solution}.
This requires finite moments through order $2D-2$, differentiability of the first $D$ moments with respect to the parameter, and nonsingular $G_D$.
The solutions are polynomials of degree at most $(D-1)$, evaluable in $O(D)$ time, independent of support size - another direct consequence of the least norm principle.

We show an example Poisson derivation below.
Furthermore, we show for continuous variables at order 2, Fini recovers classic pathwise gradients.
Going forward, we use exactness and unbiasedness interchangeably, and call an estimator order-D when it is unbiased for every polynomial of degree at most $D$.  


\paragraph{Poisson example:}
\label{sec:pois-example}
Let $Z\sim\operatorname{Pois}(\exp(\eta))$.
The order-3 Gram matrix and $b_3$ are
\begin{equation}
G_3=
\begin{pmatrix}
1 & \lambda & \lambda^2+\lambda \\
\lambda & \lambda^2+\lambda & \lambda^3+3\lambda^2+\lambda \\
\lambda^2+\lambda & \lambda^3+3\lambda^2+\lambda
& \lambda^4+6\lambda^3+7\lambda^2+\lambda
\end{pmatrix},
\qquad
b_3=
\begin{pmatrix}
\lambda\\
\lambda^2+\lambda/2\\
\lambda^3+2\lambda^2+\lambda/3
\end{pmatrix}.
\label{eq:poisson-gram}
\end{equation}
The order-2 system uses the leading $2\times2$ block of $G_3$ and the
first two entries of $b_3$.
Solving the two systems and substituting into
$w_D^\star(z)=v_D(z)^\top G_D^{-1}b_D$ yields
\begin{equation}
\boxed{
    w_2^\star(Z)=\frac{Z+\lambda}{2},
    \qquad
    w_3^\star(Z)=\frac{5\lambda}{12}
    +\frac{8\lambda+1}{12\lambda}Z
    -\frac{1}{12\lambda}Z^2.
}
\label{eq:poisson-weights}
\end{equation}
These are the Fini-2 and 3 weights used in \Cref{lst:intro-poisson}.
The construction extends to any prescribed order; \appref{app:estimators}
gives the Fini-4 weight, whose variance under the log-rate
parameterization diverges as $\lambda\downarrow0$.
This illustrates the variance cost of higher exactness.

\paragraph{Poisson MVD and quadrature.}
The same weights also follow from the difference identity
\[
\partial_\eta\E[f(Z)]
=
\lambda\E[f(Z+1)-f(Z)]
=
\lambda\E\!\left[\int_0^1 f'(Z+t)\,dt\right].
\]
Direct use of the difference identity requires an extra evaluation of
$f$ at $Z+1$.
For order 2, trapezoidal quadrature approximates the increment by
$(f'(Z)+f'(Z+1))/2$.
The Poisson shift identity
$\lambda\E[g(Z+1)]=\E[Zg(Z)]$
then replaces $f'(Z+1)$ by $(Z/\lambda)f'(Z)$ in expectation,
recovering $w_2^\star$.
The order-3 rule uses interpolation nodes at $t=0,1,2$;
\appref{app:mvd-quadrature} gives the algebra.

\paragraph{Exact Pathwise Recoveries and ST recoveries: }
\label{sec:pathwise-connections}
Fini-2 recovers fully exact pathwise gradients in the following cases:
for scalar location--scale families, $w_2^\star(Z)=\partial_\theta Z$;
for Gaussian random vectors, we recover the minimum-energy affine pathwise field.
Both estimators are exact for every sufficiently regular $f$, although we impose exactness only for quadratics.
Furthermore for categoricals, if we replace the min-norm rule with a symmetry principle, we recover ReinMax at temperature $=1$, and Fini-2 corresponds to hard one-hot MVE in \cite{hooperetal2026}. 
See \Cref{app:connections} for details. 

\section{Theory}
\label{sec:theory}
We establish uniqueness and weight-variance optimality, characterize two sources of excess variance, and bound bias outside the prescribed class.
Proofs and details are in \Cref{app:proofs} and \Cref{app:optimization}.
We start by showing our solutions are unique, and the solved weights are variance optimal.

\begin{proposition}
\label{prop:min-variance}
Let $V_D=\operatorname{span}\{1,Z,\ldots,Z^{D-1}\}$ and assume $G_D$ is
nonsingular. Every degree-$D$ exact weight $w\in L^2(p_\theta)$ satisfies
\begin{equation}
w=w_D^\star+r,\qquad r\perp V_D,\qquad
\Var_\theta[w]=\Var_\theta[w_D^\star]+\|r\|_2^2.
\label{eq:min-variance-gap}
\end{equation}
Thus $w_D^\star$ uniquely minimizes both norm and weight variance in the
exactness class.
\end{proposition}
\Cref{prop:min-variance} states that any other admissible solution must have an excess null component, which increases weight variance; given ReinMax applied to Poisson with categorical representation is order 2, we can characterize its excess weight variance.

\begin{corollary}
\label{prop:poisson-reinmax}
For $Z\sim\Pois(\lambda)$ and $\eta=\log\lambda$, the large-truncation
limit of categorical ReinMax at $\tau=1$ is degree-2 exact, with
\begin{equation}
w_{\mathrm{RM}}(Z)=\underbrace{\frac{Z+\lambda}{2}}_{w_2^\star(Z)}
+
\underbrace{
\frac{(Z-\lambda)^2-(Z-\lambda)-\lambda}{2}
}_{r_{\mathrm{RM}}(Z) \perp V_2},\qquad
\Var[w_{\mathrm{RM}}]-\Var[w_2^\star]=\frac{\lambda^2}{2}.
\label{eq:poisson-reinmax-decomposition}
\end{equation}
\end{corollary}
The gap in weight variance is due to the null component $r_{RM}$.
Its expression and proof are in \Cref{prop:poisson-reinmax-limit}, and a general theorem about Reinmax applied to scalar exponential families is in \Cref{prop:reinmax-nef}, pointing out excess variance in combining ReinMax with categorical approximation.

Rao-Blackwellization (RB) is a variance reduction technique for compound representations. 
Proposition below shows compounding can lead to excess variance, and when RB can neutralize it.

\begin{proposition}
\label{prop:rao-blackwell}
Let $p_\theta(z,u)$ have marginal $p_\theta(z)$, and let
$\widetilde w_D(U,Z)\in L^2$ be degree-$D$ exact under the joint law.
Then $\bar w_D(Z)=\E_\theta[\widetilde w_D\mid Z]$ is also degree-$D$ exact,
and
\begin{equation}
\Var_\theta[\widetilde w_D]
=
\Var_\theta[w_D^\star]
+
\E_\theta[\Var_\theta(\widetilde w_D\mid Z)]
+
\|\bar w_D-w_D^\star\|_2^2.
\label{eq:rao-blackwell-decomposition}
\end{equation}
\end{proposition}
RB removes auxiliary noise; least-norm optimization removes
the remaining null component. 
Rao--Blackwellization suffices only when $\bar w_D=w_D^\star$ 
, which we demonstrate below with the Gamma--Poisson compound representation of Negative Binomial variables. 
\begin{corollary}
\label{cor:gamma-poisson}
Let $\Lambda\sim\operatorname{Gamma}(r,\operatorname{rate}=(1-p)/p)$,
$Z\mid\Lambda\sim\Pois(\Lambda)$, and $\eta=\operatorname{logit}p$.
For the compound order-$D$ Poisson construction, conditioning its weight
on $Z$ gives
\begin{equation}
\bar w_2=w_2^\star,\qquad
\bar w_3=w_3^\star+r_3,\qquad r_3\perp V_3,\quad r_3\not\equiv0,
\label{eq:gamma-poisson-conditioning}
\end{equation}
where $w_D^\star$ is our Negative-Binomial weight.
\end{corollary}
Rao--Blackwellization closes the weight-variance gap at order 2 but not at
order 3.
The full statement and variance identities are in \Cref{app:rao-blackwell} and \Cref{app:gamma-poisson}.

With variance characterized, we bound the bias for smooth functions beyond our prescribed class.

\begin{proposition}
\label{prop:smooth-bias}
For a degree-$D$ exact weight $w$, let
$B_w(f)=\E_\theta[w(Z)f'(Z)]-\partial_\theta\E_\theta[f(Z)]$.
Fix $a\in\mathbb R$ independently of $\theta$ and write
$m_k(a)=\E_\theta[|Z-a|^k]$.
Under the smoothness, moment, and score-identity conditions in \Cref{app:bias-proof},
\begin{equation}
|B_w(f)|\le \|f^{(D+1)}\|_\infty
\left[
\frac{m_{2D}(a)^{1/2}}{D!}\|w\|_2
+
\frac{m_{2D+2}(a)^{1/2}}{(D+1)!}\|s_\theta\|_2
\right].
\label{eq:smooth-bias-bound}
\end{equation}
\end{proposition}
Exactness cancels the degree-$D$ Taylor polynomial, leaving bias controlled
by the first unmatched derivative. Least norm minimizes the bound's only
weight-dependent term. 
The bound therefore predicts small bias when the downstream computation is locally well approximated by low-degree polynomials, a regime we expect to be common for sufficiently smooth neural objectives.

\section{Related Work}
\label{sec:related-work}

\paragraph{Unbiased discrete gradient estimators.}
Score-function estimators use control variates,
multi-sample baselines, and antithetic, coupled, or Stein identities to reduce variance
\citep{
mcge_survey,nvil,rebar,relax,vimco,arm,disarm,coupled_grad,rodeo}.
MVD and GO gradients instead exploit distributional identities for lower variance, at the cost of additional $f$ evaluations \citep{mvd,mvd_vi,go}.
For Poisson, \secref{sec:pois-example} shows that quadrature of the MVD identity yields our weights.
These methods retain arbitrary-function unbiasedness but require many downstream function evaluations, whereas Fini has low bias and works with a single downstream evaluation.

\paragraph{Gradient estimators for the categorical distribution}
Categorical gradient methods provide a common interface for discrete variables, or \emph{categorical representation}, but their computation grows with the number of discrete states \citep{second_gsm_poglm}. 
Unlike categorical representation, Fini applies directly to the target law, avoiding excess compute overheads.
Gumbel--Softmax uses temperature-controlled continuous
relaxations \citep{gsm,concrete}, whereas ST estimators
preserve the hard sample and substitute its backward derivative
\citep{ste}.

ReinMax extends the ST estimator \citep{reinmax}.
It shows ST is a linear approximation, 
analogous to Euler's method in numerical quadrature. 
Given Euler's method's first-order bias can be improved to second-order, ST can also be improved similarly.
ReinMax thus replaces the linear ST by a quadratic one, and demonstrates consistently excellent empirical performance across benchmarks \citep{reinmax}, making it our principal generalist comparator.
\citet{hooperetal2026} extend this line of work by axiomatizing a broader ST family and deriving a second-order minimum-variance estimator (MVE).
Their estimator has a temperature knob that interpolates between ST and their second-order estimator.
Although theoretically MVE has lower weight variance than ReinMax, MVE does not consistently outperform ReinMax empirically.
Our order-2 solutions recover their one-hot and 1D estimators for categorical variables. 

\paragraph{Other distribution-specific estimators and compound estimators.} 
Although most discrete MC gradient estimation research focuses on the categorical distribution, many tailored estimators have been developed for distributions such as the Poisson.
Tailored estimators exploit distribution-specific structure and can be highly effective.
For Poisson, EAT-cubic differentiates a relaxed inter-arrival construction with cubic interpolation and is our principal specialist comparator \citep{knuth_poisson,pvae,eat_cubic}.
Compound representations transfer an estimator through auxiliary variables;
Gamma--Poisson, for example, reuses a Poisson estimator for a Negative-Binomial variable.
Proposition~\ref{prop:rao-blackwell} and
Corollary~\ref{cor:gamma-poisson} quantify the resulting variance overhead.
Our construction instead applies the general recipe directly to the target variable.

\section{Numerical Experiments}
\label{sec:experiments}
We test four empirical claims: 
(i) Fini is unbiased for functions inside its prescribed exactness class; 
(ii) For Poisson, low-order Fini  matches or outperforms competitors on nonlinear and hierarchical scenarios beyond their prescribed exactness class;
(iii) Fini performs well for other common discrete variables such as Negative Binomial and Binomial;
and (iv) across GPUs and benchmarks, low-order Fini consistently outspeeds all competitors. 


Our main Poisson-specific comparator is EAT-cubic, which differentiates a relaxed Poisson inter-arrival construction using cubic interpolation \citep{eat_cubic}.
Our general comparator is ReinMax, an ST estimator for categorical variables \citep{reinmax}; we apply it to count distributions (Poisson, Binomial, and Negative-Binomial) through categorical approximation where feasible and through the Gamma--Poisson representation for Negative-Binomial variables.
For unbounded categorical approximations, we choose the truncation so that the omitted tail has probability $ \leq 0.001$; the appendix gives more detail.
Both comparators have a tuned temperature $\tau$.
We also include EAT-sigmoid, Gumbel--Softmax, and score-function estimators where they are part of the original benchmark comparison. EAT-sigmoid and Gumbel--Softmax require a temperature.

We follow the metrics in prior works for the corresponding benchmarks.
Higher ELBO and lower perplexity are better.
Parenthesized numbers beside temperature-based methods give the selected $\tau$; Fini and the score-function estimator do not use a temperature.

\subsection{Linear Poisson VAE: Testing Order-2 Exactness}
\label{sec:linear-vae}

\begin{figure}[t]
    \centering
    \includegraphics[width=0.74\linewidth]{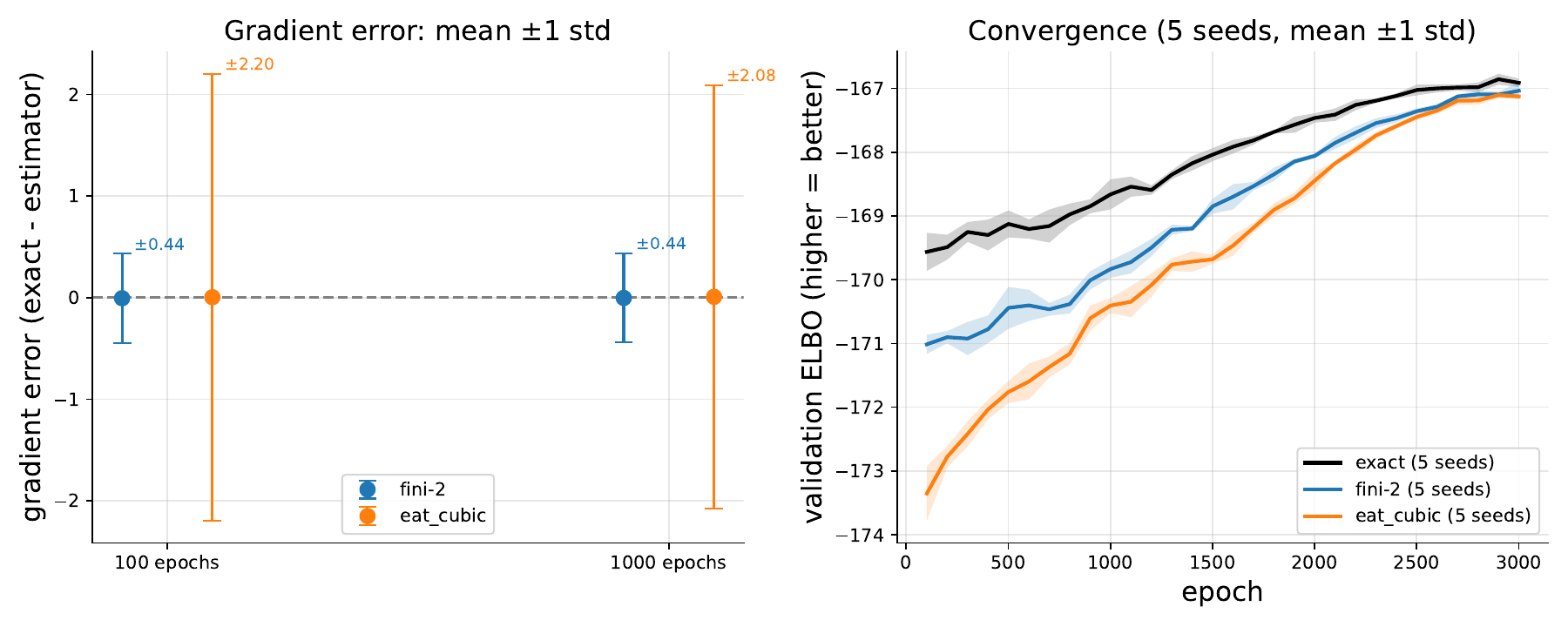}
    \captionsetup{font=small,skip=3pt}
    \caption{Linear Poisson VAE with 512 latent dimensions.
    \textbf{Left}: mean single-sample gradient error relative to the
    analytic gradient early (epoch 100) and partway through
    training (epoch 1,000), shown as dots with standard-deviation bars.
    Both means are near zero; Fini-2 has about one-fifth the standard
    deviation of tuned EAT-cubic.
    \textbf{Right}: validation ELBO over training, mean $\pm$ standard
    deviation across five matched seeds. Fini-2 follows the
    analytic-gradient trajectory more closely.}
    \label{fig:lin-vae}
\end{figure}

In a Poisson VAE, an encoder specifies a Poisson distribution over latent
counts $Z$ given an observation, and a decoder maps those counts back to an
observation distribution.
We use the linear-decoder van Hateren benchmark of \citet{eat_cubic}, which
has an analytic gradient for comparison.
With a Gaussian observation likelihood and a linear decoder, the benchmark's
downstream objective is quadratic in the latent counts, so Fini-2 is exact here.
We tune EAT-cubic over its reported temperature grid.
In Figure~\ref{fig:lin-vae}, the observed mean gradient errors at both
checkpoints are near zero, while Fini-2 has about one-fifth the gradient-error
standard deviation of tuned EAT-cubic.
Its validation-ELBO trajectory also lies closer to the analytic-gradient
reference during training.

\subsection{Performance on Realistic Benchmarks}
\label{sec:beyond-matched}

Proposition~\ref{prop:smooth-bias} bounds bias outside the prescribed class, but can low-order Fini perform well in realistic applications, where functions can be nonlinear, and have multiple stochastic layers?

\paragraph{Neural sequence model.}
The partially observed generalized linear model (POGLM) is an autoregressive latent variable model for inferring connectivities from partially observed spike trains. 
As a dynamic graphical model, it has both observed and latent Poisson variables whose rates depend on recent count histories \citep{second_gsm_poglm}. 
We reproduce the synthetic benchmark in \citet{eat_cubic}, following the same temperature grids and tuning protocol across 30 matched independent seeds.
Table~\ref{tab:poglm} compares validation ELBOs for Fini 2 and 3, EAT-cubic, Gumbel--Softmax, EAT-sigmoid, and the score-function estimator on the synthetic POGLM benchmark. The temperature-free Fini 2 and 3 have mean ELBOs close to tuned EAT-cubic, while the other baselines are worse.
\begin{table}[!htbp]
\centering
\begin{minipage}{\linewidth}
\captionsetup{font=footnotesize,skip=3pt}
\caption{Synthetic POGLM validation ELBO (higher is better), mean $\pm$ std. 
over 30 seeds. Parentheses give the selected temperature $\tau$. 
Boldface marks the three highest means, which differ by at most $0.02$.}
\label{tab:poglm}
\centering
\scriptsize
\setlength{\tabcolsep}{3.0pt}
\renewcommand{\arraystretch}{1.10}
\begin{tabular}{@{}lcccccc@{}}
\toprule
Dataset
& Fini-3
& Fini-2
& EAT-cubic (0.1)
& GSM (0.5)
& EAT-sigmoid (0.5)
& Score \\
\midrule
Synthetic
& $\mathbf{-330.61\pm0.79}$
& $\mathbf{-330.60\pm0.75}$
& $\mathbf{-330.62\pm0.79}$
& $-331.94\pm0.95$
& $-332.18 \pm 0.97$
& $-333.23 \pm 1.51$ \\
\bottomrule
\end{tabular}
\end{minipage}
\end{table}

\paragraph{Nonlinear Poisson VAEs.}
Following the discrete-VAE setup of \citet{reinmax}, we replace Bernoulli latents with Poisson latents on MNIST, Fashion-MNIST, and Omniglot.
We use dynamically binarized images with Bernoulli likelihoods and whitened continuous images with Gaussian likelihoods (\appref{app:nonlinear-count-vae-details}).
\Cref{tab:binary-latent-vae} compares Fini-2 and Fini-3 with EAT-cubic and ReinMax.
Fini-3 leads in four of six settings, including all Bernoulli-likelihood cases, with Fini-2 close behind.
Both orders occupy the leading tier across all six settings, but their ranking varies under Gaussian likelihoods: higher exactness does not uniformly improve optimization.

\begin{table}[t]
\centering
\begin{minipage}{\linewidth}
\captionsetup{font=footnotesize,skip=3pt}
\caption{Nonlinear Poisson-VAE training ELBO $\uparrow$, mean $\pm$ standard deviation over five matched seeds. Columns pair each dataset with a Bernoulli or Gaussian observation likelihood. 
First is bold, second is underlined.
}
\label{tab:binary-latent-vae}
\centering
\resizebox{\linewidth}{!}{%
\begin{tabular}{lcccccc}
\toprule
& \multicolumn{3}{c}{Bernoulli Likelihoods}
& \multicolumn{3}{c}{Gaussian Likelihoods} \\
\cmidrule(lr){2-4}
\cmidrule(lr){5-7}
& MNIST
& Fashion-MNIST
& Omniglot
& MNIST
& Fashion-MNIST
& Omniglot \\
\midrule
Fini-3
& $\bm{-118.60 \pm 0.26}$
& $\bm{-251.98 \pm 0.40}$
& $\bm{-129.38 \pm 0.07}$
& $\underline{595.48\pm3.72}$
& $\bm{133.74 \pm2.80}$
& $\underline{403.14 \pm 2.46}$ \\

Fini-2
& $\underline{-118.88 \pm 0.19}$
& $\underline{-252.26 \pm 0.18}$
& $\underline{-129.93\pm0.12}$
& $594.47\pm3.92$
& $132.09\pm3.29$
& $\bm{403.80 \pm 2.00}$ \\

EAT-cubic
& $-120.06 \pm 0.17$ (0.5)
& $-253.04 \pm 0.22$ (0.5)
& $-130.76 \pm 0.07$ (0.2)
& $\bm{595.66\pm 3.60}$ (0.5)
& $\underline{132.42\pm1.70}$ (0.5)
& $400.84\pm2.05$ (0.5)\\

ReinMax
& $-125.87 \pm 0.38$ (1.1)
& $-260.56 \pm 0.57$ (1.1)
& $-140.07 \pm 0.40$ (1.0)
& $589.21\pm2.78$ (1.0)
& $104.06\pm2.51$ (1.1)
& $396.63 \pm 1.53$ (1.0) \\
\bottomrule
\end{tabular}}
\end{minipage}
\end{table}

\paragraph{Hierarchical topic models.}
We next test our estimators on Deep Exponential Families (DEF), a deep generative model used for topic modeling and recommender systems \citep{def}.
A DEF generates observed data through multiple stochastic layers, where each layer draws from an exponential-family distribution whose natural parameters come from the previous layer.
We train amortized two-layer sparse Poisson DEFs on 20 Newsgroups (20 NG) and RCV1 using the inference architecture of \citet{nvil}; \appref{app:DEF-deets} gives the full setup. 

\begin{table}[t]
\centering
\begin{minipage}{\linewidth}
\captionsetup{font=footnotesize,skip=3pt}
\caption{Poisson DEF results on train and test ELBO ($\uparrow$) and heldout perplexity (PPL, $\downarrow$), mean $\pm$ standard deviation over five matched seeds. Best is bold; second is underlined.
}
\label{tab:topic-modeling-perplexity}
\centering
\small
\setlength{\tabcolsep}{3pt}
\renewcommand{\arraystretch}{1.06}
\resizebox{\linewidth}{!}{%
\begin{tabular}{@{}llccc@{\hspace{6pt}}ccc@{}}
\toprule
&
& \multicolumn{3}{c}{20 Newsgroups}
& \multicolumn{3}{c}{RCV1} \\
\cmidrule(lr){3-5}
\cmidrule(lr){6-8}
Model
& Training method
& Train ELBO $\uparrow$
& Test ELBO $\uparrow$
& PPL $\downarrow$
& Train ELBO $\uparrow$
& Test ELBO $\uparrow$
& PPL $\downarrow$ \\
\midrule
\multirow{4}{*}{Poisson DEF}
& Fini-3
& $\bm{-179.30 \pm 2.81}$
& $\underline{-199.31 \pm 2.04}$
& $660.7 \pm 7.6$
& $\underline{-269.27 \pm 3.16}$
& $\underline{-267.64 \pm 3.73}$
& $\bm{379.3 \pm 5.4}$ \\

& Fini-2
& $\bm{-179.30 \pm 2.62}$
& $\bm{-198.98 \pm 1.92}$
& $\underline{659.7 \pm 6.7}$
& $\bm{-269.20 \pm 3.05}$
& $\bm{-267.60 \pm 3.73}$
& $\underline{381.9 \pm 4.6}$ \\

& EAT-cubic (0.2)
& $\underline{-180.94 \pm 2.59}$
& $-199.77 \pm 2.17$
& $664.6 \pm 8.5$
& $-271.92 \pm 3.18$
& $-270.16 \pm 3.82$
& $388.4 \pm 4.1$ \\

& ReinMax (1.3)
& $-191.25 \pm 2.88$
& $-210.01 \pm 2.37$
& $\bm{638.1 \pm 10.5}$
& $-301.86 \pm 2.83$
& $-300.05 \pm 3.27$
& $384.4 \pm 4.8$ \\
\bottomrule
\end{tabular}%
}
\end{minipage}
\end{table}

\Cref{tab:topic-modeling-perplexity} shows the results.
Fini 2 and 3 have the best train and test ELBOs on both datasets. 
ReinMax ranks first on heldout perplexity for the smaller 20NG, and Fini-3 leads on RCV1.
To diagnose, we split the ELBO into log-likelihood and the KL terms, and use the low variance MVD as the ground truth.
ReinMax at \(\tau=1.3\) tilts toward worse KL and better log likelihood, which strongly correlates with perplexity; Fini closely tracks MVD on both datasets, suggesting ReinMax’s 20NG win may reflect bias rather than better gradients.
See \Cref{app:def-rate-distortion} for details. 

\paragraph{Summary.} 
Fini-2 and Fini-3 match tuned EAT-cubic on POGLM, improve most of its ELBOs,
and outperform ReinMax on every displayed ELBO. The preferred order of Fini remains task dependent.

\subsection{Generality Across Common Discrete Variables}
\label{sec:nonlinear-vae-more}

We evaluate Poisson, Binomial, and Negative-Binomial VAEs on MNIST with latent dimensions
2, 16, and 128. 
ReinMax uses categorical representation for Binomial and Poisson (with truncation). 
Overdispersion in Negative-Binomial makes direct truncation impractical, so we use the Gamma–Poisson representation.
Results are in \Cref{tab:mnist-count-vae-train}.
Fini-3 leads at higher latent dimensions, while Fini-2 leads in two of the three 2-dimensional settings - exactness order remains problem-dependent.

\begin{table}[!htbp]
\centering
\begin{minipage}{0.88\linewidth}
\captionsetup{font=footnotesize,skip=3pt}
\caption{Count-VAE training ELBO, mean $\pm$ standard deviation over five independent matched seeds. Best is bold; second is underlined.}
\label{tab:mnist-count-vae-train}
\centering
\scriptsize
\renewcommand{\arraystretch}{1.02}
\resizebox{\linewidth}{!}{%
\begin{tabular}{@{}lcccc@{}}
\toprule
Family
& Latent dim.
& Fini-2
& Fini-3
& ReinMax \\
\midrule
\multirow{3}{*}{\shortstack[l]{binomial\\$n=8, p=0.5$}}
& 2
& $\mathbf{-174.51 \pm 0.45}$
& $\underline{-176.91 \pm 1.07}$
& $-179.67 \pm 0.67$ (1.0) \\
& 16
& $\underline{-123.38 \pm 0.14}$
& $\mathbf{-123.11 \pm 0.10}$
& $-125.55 \pm 0.10$ (1.0) \\
& 128
& $\underline{-118.73 \pm 0.22}$
& $\mathbf{-118.43 \pm 0.17}$
& $-121.45 \pm 0.15$ (1.1) \\
\midrule
\multirow{3}{*}{\shortstack[l]{negbin\\$r=8, p=0.5$}}
& 2
& $\underline{-177.34 \pm 0.15}$
& $\mathbf{-177.31 \pm 0.24}$
& $-182.58 \pm 0.48$ (1.0) \\
& 16
& $\underline{-129.29 \pm 0.16}$
& $\mathbf{-129.25 \pm 0.20}$
& $-136.58 \pm 0.71$ (1.0) \\
& 128
& $\underline{-122.58 \pm 0.06}$
& $\mathbf{-122.54 \pm 0.36}$
& $-124.41 \pm 0.24$ (1.1) \\
\midrule
\multirow{3}{*}{\shortstack[l]{poisson\\$\lambda_0=0.1$}}
& 2
& $\mathbf{-190.37 \pm 0.66}$
& $\underline{-191.77 \pm 1.57}$
& $-199.33 \pm 2.57$ (1.0) \\
& 16
& $\underline{-146.74 \pm 0.62}$
& $\mathbf{-145.68 \pm 0.31}$
& $-170.60 \pm 0.39$ (1.0) \\
& 128
& $\underline{-123.20 \pm 0.30}$
& $\mathbf{-122.80 \pm 0.40}$
& $-129.99 \pm 1.02$ (1.0) \\
\bottomrule
\end{tabular}%
}
\end{minipage}
\end{table}

\paragraph{Compound representation adds statistical overhead.}
Negative-Binomial latents can reuse a Poisson estimator through an
auxiliary rate $\Lambda$.
Corollary~\ref{cor:gamma-poisson} predicts that conditioning recovers
Fini-2 exactly, while at order 3 a nonzero $Z$-only component can remain.

\begin{figure}[!htbp]
\centering
\begin{minipage}{0.82\linewidth}
\centering
\includegraphics[width=\linewidth]{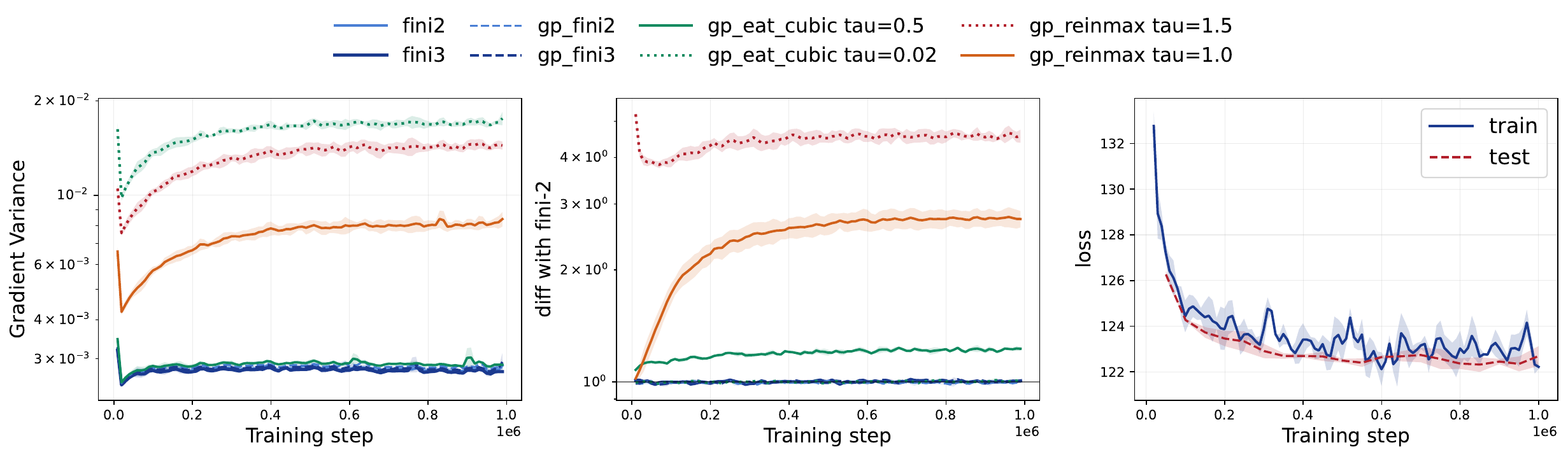}
\captionsetup{font=footnotesize,skip=3pt}
\caption{Negative-Binomial representation diagnostic.
Native Fini-2 and Fini-3 have lower gradient variance than matched
Gamma--Poisson counterparts; ReinMax has higher variance.}
\label{fig:rb-exp}
\end{minipage}
\end{figure}

We reproduce the variance study of \cite{disarm} with
Negative-Binomial latents.
The 128-dimensional variance results are in \Cref{fig:rb-exp}.
In this benchmark, each native estimator has lower gradient variance than its matched Gamma–Poisson counterpart, while ReinMax has higher variance.


\noindent
\begin{minipage}[t]{0.60\linewidth}
\vspace{0pt}
\textbf{Categorical representation adds computational overhead.}
For Binomial latents, the released hard 1D MVE implementation of \citet{hooperetal2026} operates over $K=n+1$ categorical states, whereas our fixed-order weights' costs are independent of the support size.
Figure~\ref{fig:binom-mve-support-scaling} shows the resulting support-size scaling.
The compound and categorical results therefore test different costs of non-native representations: statistical overhead from auxiliary variables and computational overhead from categorical embedding.
\end{minipage}
\hfill
\begin{minipage}[t]{0.36\linewidth}
\vspace{0pt}
\centering
\includegraphics[width=\linewidth,]{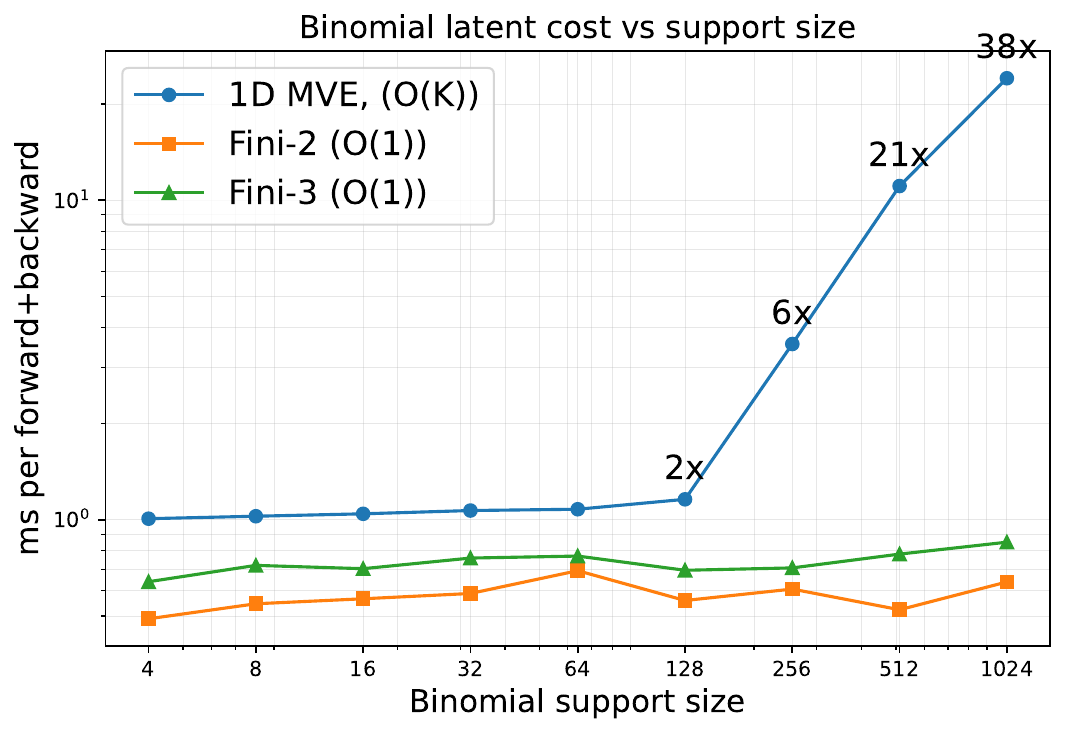}
\captionsetup{font=footnotesize,skip=3pt}
\captionof{figure}{Categorical support scaling. At $1024$, the released 1D MVE implementation takes $38\times$ as long as Fini-2. } 
\label{fig:binom-mve-support-scaling}
\end{minipage}
\par\medskip

\subsection{Speed Comparisons}
\label{sec:comp-speed}

Fini 2 and 3 require one forward--backward pass and evaluate affine or quadratic polynomials.
Across three benchmarks on L4 and A100 GPUs, Fini-2 and 3 are first and second in every workload--hardware pair (Table~\ref{tab:speed-results}); \appref{app:timing-details} gives the five-repeat post-warm-up protocol.
Fini-2 is consistently faster because its weight is affine rather than quadratic.

\begin{table}[!htbp]
\centering
\begin{minipage}{\linewidth}
\captionsetup{font=footnotesize,skip=3pt}
\caption{Wall-clock(s) per epoch ($\downarrow$), mean $\pm$ standard deviation over five repeats. }
\label{tab:speed-results}
\centering
\resizebox{\linewidth}{!}{%
\begin{tabular}{lcccccc}
\toprule
& \multicolumn{2}{c}{Linear VAE}
& \multicolumn{2}{c}{Nonlinear VAE}
& \multicolumn{2}{c}{DEF} \\
\cmidrule(lr){2-3}
\cmidrule(lr){4-5}
\cmidrule(lr){6-7}
Estimator
& L4 & A100
& L4 & A100
& L4 & A100 \\
\midrule
Fini-3
& $\underline{0.326 \pm 0.010}$
& $\underline{0.361 \pm 0.008}$
& $\underline{1.868 \pm 0.074}$
& $\underline{2.016 \pm 0.083}$
& $\underline{0.928 \pm 0.021}$
& $\underline{0.907 \pm 0.016}$ \\
Fini-2
& $\bm{0.309 \pm 0.005}$
& $\bm{0.338 \pm 0.004}$
& $\bm{1.768 \pm 0.017}$
& $\bm{1.848 \pm 0.046}$
& $\bm{0.863 \pm 0.014}$
& $\bm{0.835 \pm 0.016}$ \\
EAT-cubic
& $1.427 \pm 0.005$
& $0.581 \pm 0.010$
& $2.680 \pm 0.040$
& $2.942 \pm 0.073$
& $1.546 \pm 0.039$
& $1.565 \pm 0.015$ \\
ReinMax
& $1.144 \pm 0.006$
& $0.568 \pm 0.009$
& $2.692 \pm 0.035$
& $2.859 \pm 0.104$
& $1.451 \pm 0.024$
& $1.610 \pm 0.134$ \\
\bottomrule
\end{tabular}%
}
\end{minipage}
\end{table}

\section{Discussion}
\paragraph{Limitations: }\label{sec:limitations} Fini requires relevant moments and their parameter derivatives, and order-$D$ exactness requires at least $D$ support points.
Outside the prescribed class, Proposition~\ref{prop:smooth-bias} requires smoothness and finite moments; polynomial unbiasedness alone does not guarantee small bias.
Finally, higher $D$ adds an orthogonal correction and cannot decrease
the raw second moment; the fourth-order Poisson example in
\appref{app:poisson-order4} shows that its variance can worsen
at low rates.
Thus, estimator order controls a bias--variance tradeoff rather than
providing a uniform improvement.

\paragraph{Conclusion: } \label{sec:conclusion}We propose Fini, a recipe for extending pathwise gradients to discrete distributions by prescribing degree-$D$ polynomial exactness and then choosing the least-norm solution.
For common discrete variables, the resulting polynomial weights preserve the exact forward sample, require one 
forward--backward pass, and use neither temperature tuning nor auxiliary randomness.
Moreover, this recipe recovers classical pathwise gradients for continuous random variables.
Our theory extends the guarantees beyond the prescribed exactness class through a non-asymptotic bias bound for smooth functions and decomposes excess weight variance into auxiliary-variable noise and orthogonal null components.
Empirically, Fini-2 and 3 remain in the leading performance tier across increasingly realistic benchmarks, while ranking first and second in every runtime benchmark. 
Thus, we provide a general derivation principle for extending pathwise gradient estimators beyond continuous random variables, yielding low-bias, low-variance estimators with simple, efficient implementations.

\bibliographystyle{iclr2027_conference}
\bibliography{references}

\clearpage
\appendix
\startcontents[appendix]

\begin{center}
  \rule{\textwidth}{1.2pt}\par
  \vspace{0.8em}
  {\Large\bfseries
    Extending Pathwise Gradients to Discrete Random Variables\\
    via Finite-Order Relaxation\\[0.4em]
    Appendix\par
  }
  \vspace{0.8em}
  \rule{\textwidth}{0.4pt}
\end{center}
\vspace{0.5em}

\printcontents[appendix]{}{1}[3]{\setlength{\parskip}{0pt}}

\clearpage
\section{Discrete all-function pathwise obstruction}
\label{app:bump-proof}

\begin{proposition}[No universal exact sample pathwise weight on isolated support]
\label{prop:discrete-pathwise-obstruction}
Let $p_\theta$ be a differentiable family on a fixed discrete support
$\mathcal S\subset\mathbb R$.
Suppose that, at the parameter value under consideration, some
$z_0\in\mathcal S$ is isolated and
$\partial_\theta p_\theta(z_0)\neq0$.
Then there is no $Z$-measurable weight $w_\theta(Z)$ satisfying
\[
\E_\theta[w_\theta(Z)f'(Z)]
=
\partial_\theta\E_\theta[f(Z)]
\]
for every smooth $f$ for which the expectation exists.
\end{proposition}

\begin{proof}
Because $z_0$ is isolated, choose a smooth bump function $f$ supported in
an interval containing no other point of $\mathcal S$, with
$f(z_0)=1$ and $f'(z_0)=0$.
Then $f'(Z)=0$ almost surely, and hence
$\E_\theta[w_\theta(Z)f'(Z)]=0$ for every weight $w_\theta$.
On the other hand,
$\E_\theta[f(Z)]=p_\theta(z_0)$, so
\[
\partial_\theta\E_\theta[f(Z)]
=
\partial_\theta p_\theta(z_0)
\neq0,
\]
a contradiction.
\end{proof}
\section{Solving the Optimization Problem}
\label{app:optimization}

\begin{proposition}[Least-norm solution]
\label{prop:min-norm-solution}
Let
\[
v_D(Z)=(1,Z,\ldots,Z^{D-1})^\top,
\qquad
G_D=\E_\theta[v_D(Z)v_D(Z)^\top],
\]
and let $b_D$ denote the degree-$D$ exactness constraints, so that every
feasible weight satisfies
$\E_\theta[w(Z)v_D(Z)]=b_D$.
If $G_D$ is nonsingular, the unique least-norm degree-$D$ exact weight in $L^2(p_\theta)$ is
\begin{equation}
w_D^\star(z)
=
v_D(z)^\top G_D^{-1}b_D.
\label{eq:min-norm-solution}
\end{equation}
\end{proposition}

\begin{proof}
Let $\mathcal V_D=\operatorname{span}\{1,Z,\ldots,Z^{D-1}\}$.
Since every exactness constraint is an inner product with an element of
$\mathcal V_D$, orthogonal projection onto $\mathcal V_D$ preserves
feasibility and cannot increase the $L^2(p_\theta)$ norm.
Hence the minimum-norm weight has the form
$w_D^\star(z)=v_D(z)^\top c$.
Its exactness constraints reduce to
\[
G_Dc=b_D,
\qquad\text{hence}\qquad
c=G_D^{-1}b_D,
\]
which proves \eqref{eq:min-norm-solution}; nonsingularity of $G_D$ gives uniqueness.
\end{proof}

\subsection{Augmented representations and Rao--Blackwellization}
\label{app:rao-blackwell}

Suppose $p_\theta(z,u)$ is an augmented representation of the marginal
law $p_\theta(z)$, and let $\widetilde w_D(U,Z)\in L^2$ be any
degree-$D$ exact compound weight:
\begin{equation}
\E_\theta\!\left[\widetilde w_D(U,Z)f'(Z)\right]
=
\partial_\theta\E_\theta[f(Z)],
\qquad
f\in\mathcal P_D.
\label{eq:compound-exactness}
\end{equation}
Let
$\bar w_D(Z)=\E_\theta[\widetilde w_D(U,Z)\mid Z]$ denote its
Rao--Blackwellization.

\begin{proposition}[Rao--Blackwellization and native projection]
\label{prop:rao-blackwell-appendix}
Let $w_D^\star$ be the native minimum-$L^2(p_\theta)$ degree-$D$ exact
weight.
Then $\bar w_D$ is degree-$D$ exact and
\begin{equation}
\begin{aligned}
\|\widetilde w_D\|_2^2
&=
\|w_D^\star\|_2^2
+
\E_\theta\!\left[
\operatorname{Var}_\theta(\widetilde w_D\mid Z)
\right]
+
\|\bar w_D-w_D^\star\|_2^2, \\
\operatorname{Var}_\theta[\widetilde w_D]
&=
\operatorname{Var}_\theta[w_D^\star]
+
\E_\theta\!\left[
\operatorname{Var}_\theta(\widetilde w_D\mid Z)
\right]
+
\|\bar w_D-w_D^\star\|_2^2.
\end{aligned}
\label{eq:rao-blackwell-decomposition-appendix}
\end{equation}
Hence the native construction has no greater weight variance than the compound estimator or its Rao--Blackwellization.
If $\bar w_D\in\mathcal V_D$, then
$\bar w_D=w_D^\star$, so the native estimator is exactly the
Rao--Blackwellized compound estimator.
\end{proposition}

\begin{proof}
By the tower property,
\[
\E_\theta[\bar w_D(Z)f'(Z)]
=
\E_\theta[\widetilde w_D(U,Z)f'(Z)],
\]
so $\bar w_D$ preserves degree-$D$ exactness.
Moreover,
$\widetilde w_D-\bar w_D$ is orthogonal to every square-integrable
function of $Z$, while exactness of $\bar w_D$ and $w_D^\star$ gives
$\bar w_D-w_D^\star\perp\mathcal V_D$.
Since $w_D^\star\in\mathcal V_D$, these three components are mutually
orthogonal:
\begin{equation}
\widetilde w_D
=
w_D^\star
+
(\bar w_D-w_D^\star)
+
(\widetilde w_D-\bar w_D).
\label{eq:rao-blackwell-orthogonal-decomposition-appendix}
\end{equation}
Orthogonality gives the first identity in
\eqref{eq:rao-blackwell-decomposition-appendix}, with
\[
\|\widetilde w_D-\bar w_D\|_2^2
=
\E_\theta[
\operatorname{Var}_\theta(\widetilde w_D\mid Z)].
\]
Degree-$D$ exactness fixes the mean of all three weights through the
linear test function $f(z)=z$, giving the variance identity.
\end{proof}

\subsection{Gamma--Poisson specialization}
\label{app:gamma-poisson}
For
\[
\Lambda
\sim
\operatorname{Gamma}\!\left(
r,\operatorname{rate}=\frac{1-p}{p}
\right),
\qquad
Z\mid\Lambda\sim\operatorname{Pois}(\Lambda),
\qquad
\eta=\operatorname{logit}p,
\]
write $\Lambda=e^\eta G$ with
$G\sim\operatorname{Gamma}(r,\operatorname{rate}=1)$.
Since $\partial_\eta\log\Lambda=1$, applying the order-$D$ Poisson
log-rate weight conditionally on $\Lambda$ gives a degree-$D$ exact
compound weight.

At order 2,
$\widetilde w_2(Z,\Lambda)=(Z+\Lambda)/2$.
Conjugacy gives
\[
\Lambda\mid Z=z
\sim
\operatorname{Gamma}\!\left(
r+z,\operatorname{rate}=\frac1p
\right),
\]
and therefore
\begin{equation}
\begin{aligned}
\E[\widetilde w_2(Z,\Lambda)\mid Z=z]
&=
\frac{z+p(r+z)}{2} \\
&=
\frac{rp}{2}+\frac{1+p}{2}z
=
w_2^\star(z).
\end{aligned}
\label{eq:negbin-rao-blackwell}
\end{equation}
Thus the native order-$2$ Negative-Binomial weight is exactly the
Rao--Blackwellization of its Gamma--Poisson counterpart.

At order 3, the conditional Poisson weight can be written as
\begin{equation}
\widetilde w_3(Z,\Lambda)
=
\frac{5\Lambda}{12}
+
\frac{2Z}{3}
-
\frac{Z(Z-1)}{12\Lambda}.
\label{eq:gamma-poisson-order3-compound}
\end{equation}
Under the posterior above,
$\E[\Lambda\mid Z=z]=p(r+z)$ and, for $z\ge2$,
$\E[\Lambda^{-1}\mid Z=z]=1/[p(r+z-1)]$.
Because $z(z-1)=0$ for $z\in\{0,1\}$, conditioning
\eqref{eq:gamma-poisson-order3-compound} gives, for all $z\in\mathbb N_0$,
\begin{equation}
\bar w_3(z)
=
\frac{5p(r+z)}{12}
+
\frac{2z}{3}
-
\frac{z(z-1)}{12p(r+z-1)},
\label{eq:negbin-order3-rao-blackwell}
\end{equation}
where the final term is set to zero for $z<2$.
Both $\bar w_3$ and the native $w_3^\star$ are degree-$3$ exact, so
\[
r_3:=\bar w_3-w_3^\star\perp
\operatorname{span}\{1,Z,Z^2\}.
\]
Moreover, \eqref{eq:negbin-order3-rao-blackwell} is $O(z)$ as
$z\to\infty$, whereas the native formula in
\appref{subsec:nb-logit} satisfies
\[
w_3^\star(z)
=
-\frac{(1-p)^3}{12p(r+1)}z^2+O(z).
\]
Hence $r_3\not\equiv0$ for $r>0$ and $0<p<1$.
Proposition~\ref{prop:rao-blackwell} therefore separates the compound
order-$3$ variance gap into auxiliary conditional variance and the
additional native null-component norm $\|r_3\|_2^2$.

\section{Connection with other methods}
\label{app:connections}

\subsection{Recovering pathwise derivatives}
\label{app:pathwise}

\begin{proof}[Proof of Scalar location--scale recovery.]
Because $w_2^\star\in V_2=\operatorname{span}\{1,Z\}$, write $w_2^\star(z)=a+bz$.
Linear and quadratic exactness impose
\begin{equation}
a+b\mu_\theta=\dot\mu,
\qquad
a\mu_\theta+b(\mu_\theta^2+\sigma_\theta^2)
=
\mu_\theta\dot\mu+\sigma_\theta\dot\sigma.
\label{eq:location-scale-constraints}
\end{equation}
Thus $b=\dot\sigma/\sigma_\theta$ and $a=\dot\mu-\mu_\theta\dot\sigma/\sigma_\theta$, so
\begin{equation}
w_2^\star(z)
=
\dot\mu
+
\frac{\dot\sigma}{\sigma_\theta}(z-\mu_\theta).
\label{eq:location-scale-weight}
\end{equation}
Since $Z=\mu_\theta+\sigma_\theta\varepsilon$,
$\partial_\theta Z=\dot\mu+\dot\sigma\varepsilon=w_2^\star(Z)$, which gives exactness for every sufficiently regular $f$.
\end{proof}

For a pure location family, $\dot\sigma=0$, so degree-1 exactness already suffices.

\begin{proof}[Proof of Gaussian-vector recovery]
The gradients of quadratic polynomials span
\[
\mathcal V_2
=
\{a+A(z-\mu_\theta):a\in\mathbb R^K,\ A=A^\top\}.
\]
Hence the least-norm solution is $W_2^\star(z)=a+A(z-\mu_\theta)$ with $A=A^\top$.
Linear and quadratic exactness yield
\begin{equation}
a=\dot\mu,
\qquad
A\Sigma_\theta+\Sigma_\theta A=\dot\Sigma.
\label{eq:gaussian-vector-constraints}
\end{equation}
Because $\Sigma_\theta\succ0$, the Sylvester equation has a unique symmetric solution $A$, giving the weight results.
Gaussian integration by parts then gives
\begin{equation}
\E_\theta[W_2^\star(Z)^\top\nabla f(Z)]
=
\dot\mu^\top\E_\theta[\nabla f(Z)]
+
\frac12\operatorname{tr}\!\left(\dot\Sigma\,\E_\theta[\nabla^2f(Z)]\right)
=
\partial_\theta\E_\theta[f(Z)],
\label{eq:gaussian-full-exactness}
\end{equation}
for every sufficiently regular $f$.
\end{proof}

\subsection{Connection to measure-valued derivatives and quadrature}
\label{app:mvd-quadrature}

For $Z\sim\operatorname{Pois}(\lambda)$ with $\eta=\log\lambda$,
\[
\partial_\eta\E[f(Z)]
=
\lambda\E[f(Z+1)-f(Z)]
=
\lambda\E\!\left[\int_0^1 f'(Z+t)\,dt\right].
\]
Replacing the integral with polynomially exact quadrature and returning the shifted terms to the original Poisson law yields the same weights as the moment construction.

For degree-2 exactness, $f'$ is linear, so the two-node rule at $t=0,1$ is exact:
\begin{align}
\partial_\eta\E[f(Z)]
&=
\frac{\lambda}{2}\E[f'(Z)+f'(Z+1)]
\nonumber\\
&=
\E\!\left[\frac{\lambda+Z}{2}f'(Z)\right],
\label{eq:poisson-mvd-order2}
\end{align}
where $\lambda\E[g(Z+1)]=\E[Zg(Z)]$.
Thus quadrature recovers $w_2^\star(Z)=(\lambda+Z)/2$.

For degree-3 exactness, interpolate $f'(Z+t)$ at $t=0,1,2$.
Exact integration of quadratics on $[0,1]$ gives weights $5/12$, $2/3$, and $-1/12$:
\begin{align}
\partial_\eta\E[f(Z)]
&=
\lambda\E\!\left[
\frac{5}{12}f'(Z)
+\frac{2}{3}f'(Z+1)
-\frac{1}{12}f'(Z+2)
\right]
\nonumber\\
&=
\E\!\left[
\left(
\frac{5\lambda}{12}
+\frac{2}{3}Z
-\frac{Z(Z-1)}{12\lambda}
\right)f'(Z)
\right],
\label{eq:poisson-mvd-order3}
\end{align}
using $\lambda^2\E[g(Z+2)]=\E[Z(Z-1)g(Z)]$.
Hence
\[
w_3^\star(Z)
=
\frac{5\lambda}{12}
+\frac{8\lambda+1}{12\lambda}Z
-\frac{1}{12\lambda}Z^2,
\]
again matching the least-norm moment solution.

The identity above is the measure-valued-derivative representation of the Poisson gradient: a finite difference between neighboring laws.
The same argument applies to other distributions with finite-difference derivative identities; the Binomial quadratic weight, for example, follows from the analogous two-node construction.

\subsection{Connection to ReinMax and generalized straight-through methods}
\label{app:categorical-connections}

\begin{lemma}[Categorical quadratic exactness]
\label{lem:categorical-quadratic-exactness}
Let $Z\sim\operatorname{Cat}(p)$ be represented by
$Z\in\{e_1,\ldots,e_K\}$, where $p=\operatorname{softmax}(\eta)$
has $p_i>0$ for all $i$, and let
$J=\operatorname{diag}(p)-pp^\top$.
A same-path estimator
\[
    \widehat g(Z)=W(Z)\nabla f(Z),
    \qquad
    W_i:=W(e_i)\in\mathbb R^{K\times K},
\]
is exact for every polynomial of degree at most two if and only if
\begin{equation}
    W_i e_i=\frac12(e_i-p),
    \qquad
    p_iW_i e_j+p_jW_j e_i=0\quad(i\neq j),
    \qquad
    \sum_i p_iW_i=J .
    \label{eq:categorical-quadratic-constraints}
\end{equation}
\end{lemma}

\begin{proof}
Write
$f(z)=a^\top z+\frac12z^\top Hz$ with $H=H^\top$.
Quadratic exactness is equivalent to
\begin{equation}
    \sum_i p_iW_i(a+He_i)
    =
    J\left(a+\frac12\operatorname{diag}(H)\right)
    \qquad
    \text{for all }a\text{ and }H=H^\top .
    \label{eq:categorical-quadratic-general}
\end{equation}
The coefficient of $a$ gives the final constraint in
\eqref{eq:categorical-quadratic-constraints}. Taking
$H=e_ie_i^\top$ and
$H=e_ie_j^\top+e_je_i^\top$ gives the first two constraints.
Conversely, these matrices span the symmetric matrices, so
\eqref{eq:categorical-quadratic-constraints} implies
\eqref{eq:categorical-quadratic-general}.
\end{proof}

\begin{proposition}[Symmetry characterizes ReinMax]
\label{prop:categorical-reinmax}
Let $W$ be quadratic-exact under the assumptions of
Lemma~\ref{lem:categorical-quadratic-exactness}.
Then
\[
    W_i=W_i^\top
    \quad\text{for every }i
\]
if and only if
\begin{equation}
    W_i
    =
    \frac12
    \left[
        J+(e_i-p)(e_i-p)^\top
    \right],
    \qquad i=1,\ldots,K .
    \label{eq:reinmax-categorical-weight}
\end{equation}
Thus sample-wise symmetry uniquely selects the ReinMax weight at $\tau=1$ within the quadratic-exact class.
\end{proposition}

\begin{proof}
Write $d_i=e_i-p$ and denote the right-hand side of
\eqref{eq:reinmax-categorical-weight} by $W_i^{\rm RM}$.
Since $Je_i=p_i d_i$ and
$\sum_i p_i d_i d_i^\top=J$, its quadratic-exactness follows from
Lemma~\ref{lem:categorical-quadratic-exactness}:
\begin{align}
    W_i^{\rm RM}e_i
    &=
    \frac12\left[p_i d_i+(1-p_i)d_i\right]
    =\frac12d_i,
    \nonumber\\
    W_i^{\rm RM}e_j
    &=
    \frac{p_j}{2}(e_j-e_i),
    \qquad i\neq j,
    \nonumber\\
    \sum_i p_iW_i^{\rm RM}
    &=
    \frac12
    \left[
        J+\sum_i p_i d_i d_i^\top
    \right]
    =J .
    \label{eq:reinmax-exactness-check}
\end{align}
The matrices $W_i^{\rm RM}$ are symmetric, so it remains to show
uniqueness under symmetry.

Let $R_i:=W_i-W_i^{\rm RM}$ and define the linear matrix-valued map
\[
    \mathcal R(u):=\sum_i p_i u_iR_i,
    \qquad u\in\mathbb R^K .
\]
Both $W$ and $W^{\rm RM}$ are quadratic-exact. Applying their
difference to
$f_u(z)=\frac12(u^\top z)^2$ gives
$\mathcal R(u)u=0$ for every $u$. Polarization therefore gives
\begin{equation}
    \mathcal R(u)v=-\mathcal R(v)u
    \qquad
    \text{for all }u,v\in\mathbb R^K .
    \label{eq:reinmax-residual-skew}
\end{equation}
Since $R_i=R_i^\top$, each $\mathcal R(u)$ is symmetric. Hence, for
arbitrary $r,u,v\in\mathbb R^K$,
\begin{align}
    r^\top\mathcal R(u)v
    &=
    v^\top\mathcal R(u)r
    =
    -v^\top\mathcal R(r)u
    =
    -u^\top\mathcal R(r)v
    \nonumber\\
    &=
    u^\top\mathcal R(v)r
    =
    r^\top\mathcal R(v)u
    =
    -r^\top\mathcal R(u)v .
    \label{eq:reinmax-residual-zero}
\end{align}
Thus $\mathcal R(u)v=0$ for all $u,v$. Taking $u=e_i$ gives
$p_iR_iv=0$ for every $v$, hence $R_i=0$ because $p_i>0$.
Therefore $W_i=W_i^{\rm RM}$ for every $i$.
\end{proof}

\begin{proposition}[Least-norm categorical quadratic estimator]
\label{prop:categorical-least-norm}
Let $Z\sim\operatorname{Cat}(p)$ be represented by
$Z\in\{e_1,\ldots,e_K\}$, where $p=\operatorname{softmax}(\eta)$
has $p_i>0$ for all $i$, and let
$J=\operatorname{diag}(p)-pp^\top$.
Consider quadratic-exact same-path estimators
\[
    \widehat g(Z)=W(Z)\nabla f(Z),
    \qquad
    W_i:=W(e_i)\in\mathbb R^{K\times K}.
\]
Among all such weights, there is a unique minimizer of
$\mathbb E\|W(Z)\|_F^2$.

For $i\neq j$, define
$\omega_{ij}=p_ip_j/(p_i+p_j)$, and let
$G\in\mathbb R^{K\times K}$ have entries
\[
    G_{ij}=-\omega_{ij}\quad(i\neq j),
    \qquad
    G_{ii}=\sum_{j\neq i}\omega_{ij}.
\]
If $C\in\mathbb R^{K\times K}$ satisfies
\[
    GC=\frac12J,
\]
then the minimum-norm weight is
\[
    W_i^\star e_j
    =
    \begin{cases}
        \dfrac12(e_i-p),
        & j=i, \\[2mm]
        \dfrac{p_j}{p_i+p_j}
        C^\top(e_j-e_i),
        & j\neq i .
    \end{cases}
\]
In particular, one may take
$C=\frac12G^\dagger J$.
\end{proposition}

\begin{proof}
By Lemma~\ref{lem:categorical-quadratic-exactness}, quadratic
exactness imposes \eqref{eq:categorical-quadratic-constraints}.
The first relation fixes the diagonal columns, while the second
identifies the two off-diagonal columns associated with each pair
$i<j$. We therefore use $(W_i e_j)_{i<j}$ as coordinates for the
remaining degrees of freedom and equip them with the inner product
\[
    \langle W,V\rangle_{\mathcal H}
    :=
    \sum_{i<j}
    \frac{p_i(p_i+p_j)}{p_j}
    (W_i e_j)^\top(V_i e_j).
\]
Indeed, using
$p_jW_j e_i=-p_iW_i e_j$, the original objective reduces to
\begin{equation}
    \mathbb E\|W(Z)\|_F^2
    =
    \mathrm{const}
    +
    \sum_{i<j}
    \frac{p_i(p_i+p_j)}{p_j}\|W_i e_j\|^2
    =
    \mathrm{const}+\|W\|_{\mathcal H}^2 .
    \label{eq:categorical-norm-reduction}
\end{equation}

The remaining exactness constraints define a linear map
$\mathcal A:\mathcal H\to\mathbb R^{K\times K}$ by
\begin{equation}
    e_j^\top\mathcal AW
    :=
    \left(
        \sum_{i<j}p_iW_i e_j
        -
        \sum_{i>j}p_jW_j e_i
    \right)^\top .
    \label{eq:categorical-exactness-operator}
\end{equation}
After subtracting the fixed diagonal contribution in
\eqref{eq:categorical-quadratic-constraints}, quadratic exactness is
$\mathcal AW=\frac12J$.

For $C\in\mathbb R^{K\times K}$, the adjoint is determined by
\begin{align}
    \langle \mathcal AW,C\rangle_F
    &=
    \sum_{i<j}
    p_i(W_i e_j)^\top C^\top(e_j-e_i),
    \nonumber\\
    (\mathcal A^\ast C)_i e_j
    &=
    \frac{p_j}{p_i+p_j}
    C^\top(e_j-e_i),
    \qquad i<j .
    \label{eq:categorical-adjoint}
\end{align}
Consequently, the normal operator is
\begin{equation}
    e_j^\top\mathcal A\mathcal A^\ast C
    =
    \sum_{i\neq j}
    \omega_{ij}(e_j-e_i)^\top C
    =
    e_j^\top GC,
    \qquad
    \mathcal A\mathcal A^\ast C=GC .
    \label{eq:categorical-normal-operator}
\end{equation}

The minimum-norm solution lies in
$\operatorname{range}(\mathcal A^\ast)
=\ker(\mathcal A)^\perp$.
Hence $W^\star=\mathcal A^\ast C$, where the coefficient matrix
satisfies the normal system
\[
    GC=\frac12J .
\]
Equation~\eqref{eq:categorical-adjoint} gives the off-diagonal columns displayed in Proposition~\ref{prop:categorical-least-norm}, and the first constraint in
\eqref{eq:categorical-quadratic-constraints} gives the diagonal
columns. Since $G\mathbf 1=0$ and $J\mathbf 1=0$, one may take
$C=\frac12G^\dagger J$; adding $\mathbf 1a^\top$ to $C$ leaves
$C^\top(e_j-e_i)$ unchanged. Uniqueness follows from uniqueness of
the minimum-norm element of the feasible affine subspace.
\end{proof}

\begin{corollary}[Relation to Hooper and Shekhovtsov]
\label{cor:categorical-hooper}
The least-norm quadratic-exact estimator in
Proposition~\ref{prop:categorical-least-norm}
coincides with the $\tau\to\infty$ categorical minimum-variance estimator of \citet{hooperetal2026}.
\end{corollary}

\paragraph{Relation to the 1D MVE of Hooper and Shekhovtsov.}
Let $Z$ have a finite scalar support, with
$m_1=\E_\theta[Z]$, $m_2=\E_\theta[Z^2]$, and
$v=m_2-m_1^2>0$.
Then the degree-$2$ least-norm weight coincides with the hard
quadratically-unbiased 1D MVE of \citet{hooperetal2026}.

The degree-2 least-norm construction gives $w_2^\star(z)=a+bz$, with
\[
a+bm_1=\dot m_1,
\qquad
am_1+bm_2=\frac12\dot m_2 .
\]
Hence
\[
w_2^\star(z)
=
\frac{m_2-zm_1}{v}\,\dot m_1
+
\frac{z-m_1}{2v}\,\dot m_2 ,
\]
which corresponds to \citet[Prop.~C.3]{hooperetal2026}.
The correspondence is specific to order 2. For a scalar finite-support law, the construction extends to degree-$D$ exactness whenever the support contains at least $D$ distinct points.

\section{Theory section Proofs}
\label{app:proofs}

Throughout this appendix, a dot denotes differentiation with respect to the
parameter coordinate under consideration.

\subsection{Proof of Proposition~\ref{prop:min-variance}}
\label{app:min-variance-proof}

\begin{proof}
Let $v_D(Z)=(1,Z,\ldots,Z^{D-1})^\top$ and
$r=w-w_D^\star$. Since $w$ and $w_D^\star$ satisfy the same exactness
constraints, $\E_\theta[r(Z)v_D(Z)]=0$, hence
$r\perp V_D$. Since $w_D^\star\in V_D$,
\begin{equation}
\begin{aligned}
\E_\theta[w(Z)^2]
&=
\E_\theta[w_D^\star(Z)^2]+\E_\theta[r(Z)^2]
=
b_D^\top G_D^{-1}b_D+\E_\theta[r(Z)^2], \\
\operatorname{Var}_\theta[w(Z)]
&=
\operatorname{Var}_\theta[w_D^\star(Z)]
+\E_\theta[r(Z)^2].
\end{aligned}
\label{eq:min-variance-proof}
\end{equation}
The second identity follows because degree-$D$ exactness fixes
$\E_\theta[w(Z)]$. Equality in either objective requires $r=0$ almost
surely, proving uniqueness.
\end{proof}

\subsection{Proof of Proposition~\ref{prop:smooth-bias}}
\label{app:bias-proof}
Assume $f\in C^{D+1}$ with bounded $f^{(D+1)}$,
$m_{2D+2}(a)<\infty$, $w,s_\theta\in L^2(p_\theta)$,
and that the score identity holds for $f$ and degree-$D$ polynomials.
Let $m_k(a)=\E_\theta[|Z-a|^k]$. The constants in
Proposition~\ref{prop:smooth-bias} are
\begin{equation}
C_{D,w}(a)
=
\frac{m_{2D}(a)^{1/2}}{D!},
\qquad
C_{D,s}(a)
=
\frac{m_{2D+2}(a)^{1/2}}{(D+1)!}.
\label{eq:smooth-bias-constants}
\end{equation}

\begin{proof}
Fix an expansion point $a$ independent of $\theta$, and let $T_D$ be the
degree-$D$ Taylor polynomial of $f$ about $a$.
Degree-$D$ exactness annihilates $T_D$, so the score identity reduces the
bias to
\begin{equation}
B_w(f)
=
\E_\theta\!\left[
w(Z)\bigl(f'(Z)-T_D'(Z)\bigr)
-
s_\theta(Z)\bigl(f(Z)-T_D(Z)\bigr)
\right].
\label{eq:smooth-bias-remainder}
\end{equation}
Taylor's theorem gives
$|f'(z)-T_D'(z)|
\le
\lVert f^{(D+1)}\rVert_\infty |z-a|^D/D!$
and
$|f(z)-T_D(z)|
\le
\lVert f^{(D+1)}\rVert_\infty |z-a|^{D+1}/(D+1)!$.
Cauchy--Schwarz therefore yields
\begin{equation}
|B_w(f)|
\le
\lVert f^{(D+1)}\rVert_\infty
\left[
\frac{m_{2D}(a)^{1/2}}{D!}\lVert w\rVert_2
+
\frac{m_{2D+2}(a)^{1/2}}{(D+1)!}\lVert s_\theta\rVert_2
\right].
\label{eq:smooth-bias-proof}
\end{equation}
\end{proof}

\subsection{Order-$2$ ReinMax in a scalar natural exponential family}
\label{app:reinmax-nef}

The finite-support result isolates the component introduced by ReinMax
before specializing to Poisson.

\begin{proposition}[ReinMax excess component]
\label{prop:reinmax-nef}
Let
$p_\eta(z)=h(z)\exp\{\eta z-A(\eta)\}$
be a scalar natural exponential family represented categorically over a
finite support. Write $Y=Z-\mu$,
$V=\E_\eta[Y^2]$, $M_3=\E_\eta[Y^3]$, and
$M_4=\E_\eta[Y^4]$. At $\tau=1$,
\begin{equation}
\begin{aligned}
w_{\mathrm{RM}}(Z)
&=
\frac12(V+Y^2)
=
\underbrace{
V+\frac{M_3}{2V}Y
}_{w_2^\star(Z)}
+
\underbrace{
\frac12
\left(
Y^2-\frac{M_3}{V}Y-V
\right)
}_{r_{\mathrm{RM}}(Z)}, \\
r_{\mathrm{RM}}
&\perp\operatorname{span}\{1,Z\},
\qquad
\operatorname{Var}_\eta[w_{\mathrm{RM}}]
-
\operatorname{Var}_\eta[w_2^\star]
=
\frac14
\left(
M_4-V^2-\frac{M_3^2}{V}
\right).
\end{aligned}
\label{eq:nef-reinmax-decomposition}
\end{equation}
The variance gap is strictly positive whenever the law has at least three
support points of positive probability.
\end{proposition}

\begin{proof}
Let $C\sim\operatorname{Cat}(p)$, let $x$ contain the support values, and
write $Z=x^\top C$. With
$J(p)=\operatorname{diag}(p)-pp^\top$, the $\tau=1$ ReinMax matrix weight reduces to
\begin{equation}
A_C
=
\frac12
\left[
J(p)+(C-p)(C-p)^\top
\right],
\qquad
x^\top A_Cx
=
\frac12(V+Y^2).
\label{eq:reinmax-scalar-weight}
\end{equation}
For a natural exponential family,
$\dot\mu=V$ and $\dot V=M_3$, so the native order-$2$ weight is
$w_2^\star=V+(M_3/(2V))Y$. The residual in
\eqref{eq:nef-reinmax-decomposition} satisfies
$\E_\eta[r_{\mathrm{RM}}]
=
\E_\eta[Yr_{\mathrm{RM}}]=0$, and Proposition~\ref{prop:min-variance}
gives
\begin{equation}
\operatorname{Var}_\eta[w_{\mathrm{RM}}]
-
\operatorname{Var}_\eta[w_2^\star]
=
\lVert r_{\mathrm{RM}}\rVert_2^2
=
\frac14
\left(
M_4-V^2-\frac{M_3^2}{V}
\right).
\label{eq:nef-reinmax-gap}
\end{equation}
The gap can vanish only if
$Y^2-(M_3/V)Y-V=0$ almost surely; a nonzero quadratic cannot vanish on
three support points.
\end{proof}

\subsection{Poisson ReinMax under increasing categorical truncation}
\label{app:reinmax-poisson}

The Poisson result follows by letting the categorical truncation grow to infinite support.

\begin{proposition}[Poisson categorical limit]
\label{prop:poisson-reinmax-limit}
Let $Z_K$ have the Poisson law truncated to
$\{0,\ldots,K\}$ and renormalized, with mean $\mu_K$ and variance $V_K$.
Under categorical embedding with logits
$\ell_k=k\eta-\log(k!)$, $\eta=\log\lambda$, the $\tau=1$ ReinMax weight is
\[
w_{\mathrm{RM},K}(Z_K)
=
\frac12
\left[
V_K+(Z_K-\mu_K)^2
\right].
\]
Extend $w_{\mathrm{RM},K}(z)=\tfrac12[V_K+(z-\mu_K)^2]$ to $z\in\mathbb N_0$.
Then $w_{\mathrm{RM},K}\to w_{\mathrm{RM}}$ in $L^2(p_\lambda)$, where
$w_{\mathrm{RM}}(z)=\tfrac12[\lambda+(z-\lambda)^2]$.
The limiting weight satisfies
\begin{equation}
\begin{aligned}
w_{\mathrm{RM}}(Z)
&=
\underbrace{\frac{Z+\lambda}{2}}_{w_2^\star(Z)}
+
\underbrace{
\frac12
\left[
(Z-\lambda)^2-(Z-\lambda)-\lambda
\right]
}_{r_{\mathrm{RM}}(Z)}, \\
r_{\mathrm{RM}}
&\perp\operatorname{span}\{1,Z\},
\qquad
\lVert r_{\mathrm{RM}}\rVert_2^2
=
\frac{\lambda^2}{2}.
\end{aligned}
\label{eq:poisson-reinmax-limit-result}
\end{equation}
\end{proposition}

\begin{proof}
Proposition~\ref{prop:reinmax-nef} applied to the truncated Poisson law
gives
$w_{\mathrm{RM},K}
=
\tfrac12[V_K+(Z_K-\mu_K)^2]$.
The truncated moments give $\mu_K\to\lambda$ and $V_K\to\lambda$.
Hence the coefficients of the quadratic weight functions converge; since the Poisson law has a finite fourth moment, the convergence holds in $L^2(p_\lambda)$.

For $Y=Z-\lambda$, the limiting decomposition in
\eqref{eq:poisson-reinmax-limit-result} has
$\E[Y]=0$, $\E[Y^2]=\E[Y^3]=\lambda$, and
$\E[Y^4]=\lambda+3\lambda^2$. Therefore
\begin{equation}
\E[r_{\mathrm{RM}}]
=
\E[Zr_{\mathrm{RM}}]
=
0,
\qquad
\lVert r_{\mathrm{RM}}\rVert_2^2
=
\frac14
\E[(Y^2-Y-\lambda)^2]
=
\frac{\lambda^2}{2}.
\label{eq:poisson-reinmax-residual}
\end{equation}
\end{proof}
\section{Additional closed-form estimators}
\label{app:estimators}

For a scalar discrete variable $X$ and parameter coordinate $\theta$, write the order-$m$ same-path estimator as
\begin{equation}
    \widehat{\partial_\theta \mathbb{E}[f(X)]}
    = v_\theta^{(m)}(X)\,f'(X).
\end{equation}
Here order $m$ means exactness for every polynomial $f$ of degree at most $m$.
The formulas below are the minimum-$L^2(p_\theta)$ solutions in the stated parameter coordinates: order 2 is affine, $v^{(2)}(x)=a_0+a_1x$, and order 3 is quadratic, $v^{(3)}(x)=a_0+a_1x+a_2x^2$.

\subsection{Poisson: order 4}
\label{app:poisson-order4}
For $Z\sim\operatorname{Pois}(\lambda)$ with $\eta=\log\lambda$, order-$4$
exactness gives
\[
    w_4^\star(z)=v_4(z)^\top G_4^{-1}b_4,
    \qquad
    v_4(z)=(1,z,z^2,z^3)^\top .
\]
Using moments through order 6 and $(b_4)_d=d^{-1}\partial_\eta\E[Z^d]$, the $4\times4$ system gives
\[
    G_4^{-1}b_4
    =
    \begin{pmatrix}
        3\lambda/8\\[1mm]
        (19\lambda^2+5\lambda+2)/(24\lambda^2)\\[1mm]
        -(5\lambda+3)/(24\lambda^2)\\[1mm]
        1/(24\lambda^2)
    \end{pmatrix}.
\]
The resulting weight and variance are
\[
    \boxed{
    w_4^\star(Z)
    =
    \frac{3\lambda}{8}
    +\frac{19\lambda^2+5\lambda+2}{24\lambda^2}Z
    -\frac{5\lambda+3}{24\lambda^2}Z^2
    +\frac{1}{24\lambda^2}Z^3
    \qquad
    \operatorname{Var}(w_4^\star)=\frac{\lambda}{4}
    + \frac{1}{72} + \frac{1}{96\lambda}
    }
\]

\subsection{Binomial: $\eta=\operatorname{logit}p$}
\label{subsec:binomial-logit}

Let
\begin{equation}
    X\sim\operatorname{Binomial}(n,p),
    \qquad q=1-p,
    \qquad \eta=\operatorname{logit}p.
\end{equation}
The order-2 estimator is
\begin{equation}
    \boxed{
    v_\eta^{(2)}(x)
    =
    \frac{np}{2}
    +
    \frac{1-2p}{2}\,x.}
\end{equation}
For $n\ge 2$, the order-3 estimator is
\begin{equation}
    \boxed{
    \begin{aligned}
    v_\eta^{(3)}(x)
    ={}&
    \frac{np(5-6p)}{12q}
    \\
    &+
    \left[
        \frac{1-2p}{2}
        +
        \frac{2np+1-2p}{12(n-1)pq}
    \right]x
    \\
    &-
    \frac{1}{12(n-1)pq}\,x^2.
    \end{aligned}}
\end{equation}
The corresponding raw-$p$ weight follows from parameterization equivariance:
\begin{equation}
    v_p^{(m)}(x)=\frac{v_\eta^{(m)}(x)}{pq}.
\end{equation}

\subsection{Negative binomial: $\eta=\operatorname{logit}p$}
\label{subsec:nb-logit}

We use the convention
\begin{equation}
    \Pr(X=x)
    =
    \frac{\Gamma(r+x)}{\Gamma(r)x!}
    (1-p)^r p^x,
    \qquad
    q=1-p,
\end{equation}
so that $\mathbb{E}[X]=rp/q$. For $\eta=\operatorname{logit}p$, the order-2 estimator is
\begin{equation}
    \boxed{
    v_\eta^{(2)}(x)
    =
    \frac{rp}{2}
    +
    \frac{1+p}{2}\,x.}
\end{equation}
The order-3 estimator is
\begin{equation}
    \boxed{
    \begin{aligned}
    v_\eta^{(3)}(x)
    ={}&
    \frac{rp(p+5)}{12}
    \\
    &+
    \left[
        \frac{1+p}{2}
        +
        \frac{q^2(2rp+1+p)}{12p(r+1)}
    \right]x
    \\
    &-
    \frac{q^3}{12p(r+1)}\,x^2.
    \end{aligned}}
\end{equation}
The corresponding raw-$p$ weight is
\begin{equation}
    v_p^{(m)}(x)=\frac{v_\eta^{(m)}(x)}{pq}.
\end{equation}

\subsection{Negative binomial: $\rho=\log r$}
\label{subsec:nb-logr}

Under the same negative-binomial convention, let
\begin{equation}
    \rho=\log r.
\end{equation}
The order-2 estimator is
\begin{equation}
    \boxed{
    v_\rho^{(2)}(x)
    =
    \frac{rp}{2q}
    +
    \frac{1}{2}\,x.}
\end{equation}
The order-3 estimator is
\begin{equation}
    \boxed{
    \begin{aligned}
    v_\rho^{(3)}(x)
    ={}&
    \frac{rp(5-p)}{12q}
    \\
    &+
    \left[
        \frac{1}{2}
        +
        \frac{(1+p)(2rp+1+p)}{12p(r+1)}
    \right]x
    \\
    &-
    \frac{q(1+p)}{12p(r+1)}\,x^2.
    \end{aligned}}
\end{equation}
The corresponding raw-$r$ weight is
\begin{equation}
    v_r^{(m)}(x)=\frac{v_\rho^{(m)}(x)}{r}.
\end{equation}

\subsection{Beta-binomial: $\eta=\operatorname{logit}\pi$}
\label{subsec:betabin-logit}

Let
\begin{equation}
    X\sim\operatorname{BetaBinomial}(n,\alpha,\beta),
    \qquad
    \alpha=\kappa\pi,
    \qquad
    \beta=\kappa(1-\pi),
\end{equation}
and write
\begin{equation}
    q=1-\pi,
    \qquad
    \eta=\operatorname{logit}\pi.
\end{equation}
The order-2 estimator is
\begin{equation}
    \boxed{
    v_\eta^{(2)}(x)
    =
    \frac{n\pi}{2}
    +
    \frac{1-2\pi}{2}\,x.}
\end{equation}

For order 3, define
\begin{align}
    S
    &:=
    \frac{
        2\kappa n\pi-2\kappa\pi+\kappa+2n
    }{\kappa+2},
    \\
    T
    &:=
    \frac{
        \kappa n\pi(n-1)(\kappa\pi+1)
    }{(\kappa+1)(\kappa+2)},
\end{align}
and
\begin{equation}
    C_\eta
    :=
    -
    \frac{
        (\kappa+1)(\kappa+2)(\kappa+3)(\kappa+2n)
    }{
        12\kappa(n-1)
        (\kappa\pi+1)
        (\kappa q+1)
        (\kappa+n+1)
    }.
\end{equation}
Then
\begin{equation}
    \boxed{
    \begin{aligned}
    v_\eta^{(3)}(x)
    ={}&
    \left(\frac{n\pi}{2}+C_\eta T\right)
    \\
    &+
    \left(\frac{1-2\pi}{2}-C_\eta S\right)x
    \\
    &+
    C_\eta x^2.
    \end{aligned}}
\end{equation}
The corresponding raw-$\pi$ weight is
\begin{equation}
    v_\pi^{(m)}(x)
    =
    \frac{v_\eta^{(m)}(x)}{\pi(1-\pi)}.
\end{equation}

\subsection{Beta-binomial: $\rho=\log\kappa$}
\label{subsec:betabin-logkappa}

Under the same mean--concentration parameterization, let
\begin{equation}
    \rho=\log\kappa.
\end{equation}
Retain the definitions of $q$, $S$, and $T$ from
Section~\ref{subsec:betabin-logit}, and define
\begin{equation}
    D
    :=
    \frac{
        \kappa(n-1)
    }{
        2(\kappa+1)(\kappa+n)
    }.
\end{equation}
The order-2 estimator is
\begin{equation}
    \boxed{
    v_\rho^{(2)}(x)
    =
    Dn\pi-Dx.}
\end{equation}
For order 3, define
\begin{equation}
    C_\rho
    :=
    \frac{
        \kappa(\kappa+3)(2\pi-1)(3\kappa+2n+2)
    }{
        12(\kappa+n)
        (\kappa\pi+1)
        (\kappa q+1)
        (\kappa+n+1)
    }.
\end{equation}
Then
\begin{equation}
    \boxed{
    \begin{aligned}
    v_\rho^{(3)}(x)
    ={}&
    \left(Dn\pi+C_\rho T\right)
    \\
    &+
    \left(-D-C_\rho S\right)x
    \\
    &+
    C_\rho x^2.
    \end{aligned}}
\end{equation}
The corresponding raw-$\kappa$ weight is
\begin{equation}
    v_\kappa^{(m)}(x)
    =
    \frac{v_\rho^{(m)}(x)}{\kappa}.
\end{equation}

\section{Two-Parameter Straight-Through Implementation}
\label{sec:two-parameter-st}

Consider a hard scalar sample $Z_{\mathrm{hard}}\sim p_{a,b}$ with optimized coordinates $a$ and $b$. Suppose the desired same-path estimators are
\begin{align}
    \partial_a\mathbb{E}[f(Z)]
    &\approx
    \mathbb{E}\!\left[u(Z)f'(Z)\right],
    \\
    \partial_b\mathbb{E}[f(Z)]
    &\approx
    \mathbb{E}\!\left[v(Z)f'(Z)\right].
\end{align}
Here $u$ and $v$ may be any minimum-norm weights above, expressed in the model's optimized coordinates $a$ and $b$.

Let $\operatorname{sg}(\cdot)$ denote stop-gradient. The corresponding straight-through surrogate is
\begin{equation}
    \boxed{
    Z_{\mathrm{st}}
    =
    \operatorname{sg}(Z_{\mathrm{hard}})
    +
    \operatorname{sg}(u(Z_{\mathrm{hard}}))
    \bigl(a-\operatorname{sg}(a)\bigr)
    +
    \operatorname{sg}(v(Z_{\mathrm{hard}}))
    \bigl(b-\operatorname{sg}(b)\bigr).}
\end{equation}

For a Negative-Binomial variable parameterized by
\begin{equation}
    a=\eta=\operatorname{logit}p,
    \qquad
    b=\rho=\log r,
\end{equation}
one evaluates
\begin{equation}
    u(Z_{\mathrm{hard}})=v_\eta^{(m)}(Z_{\mathrm{hard}}),
    \qquad
    v(Z_{\mathrm{hard}})=v_\rho^{(m)}(Z_{\mathrm{hard}}),
\end{equation}
and inserts both weights into the same straight-through expression.

In PyTorch pseudocode, the scalar two-parameter construction is
\begin{verbatim}
z_st = (
    z_hard.detach()
    + u.detach() * (a - a.detach())
    + v.detach() * (b - b.detach())
)
\end{verbatim}

\section{Additional Experimental Details}
\label{app:experiments}

Unless stated otherwise, means and standard deviations are computed over independently seeded runs, with seeds matched across estimator configurations. 
One additional detail on the tested method, is that for ReinMax, we take the largest rate of the batch, and calculate truncation using that rate and probability $p=0.999$.

\subsection{Linear Poisson VAE}
\label{app:linear-vae-details}
We follow the benchmark in Appendix G.1 of \citet{eat_cubic}.
EAT-cubic is tuned over the reported temperature grid $[0.02,0.05,0.2]$; $0.05$ performs best in our replication.
Across five independent matched seeds, order 2 obtains $-167.02\pm0.05$, EAT-cubic obtains $-167.11\pm0.06$ at temperature $0.05$ ($p=0.034$, two-sided Welch $t$-test), and the exact reference gives $-166.99\pm0.04$.
Because final ELBO saturates in this controlled benchmark, we use it primarily to validate exactness and gradient variance rather than as a headline performance comparison.

\subsection{Neural sequence model}
\label{app:poglm-details}
We follow the synthetic spike-train benchmark, temperature grids, and tuning protocol in Appendix G.2 of \citet{eat_cubic}, except for initialization.
In our implementation, the stated $N(0,0.5^2)$ initialization did not reproduce the reported ELBO range: with either one or five Monte Carlo samples, runs converged around $-427$ to $-429$.
The repository-default \texttt{Conv1d} initialization reproduced results near $-330.5$, so we used it for every estimator and held it fixed across methods.
Reported statistics use 30 independent matched seeds.

\subsection{Nonlinear Poisson and count VAEs}
\label{app:nonlinear-count-vae-details}
We reuse the experimental design of \citet{double_cv}, which follows earlier discrete-gradient control-variate studies \citep{arm,disarm,arsm}, and replace only the latent distribution with Poisson.
The prior is $\operatorname{Pois}(1.0)$: no proper uninformative Poisson prior exists, and rate $1.0$ corresponds to natural parameter $0$.
For the count-family comparison, we retain the same design but restrict the experiments to dynamically binarized MNIST.
Binomial and Negative-Binomial latents use $p=0.5$, while the sparse Poisson setup uses $\lambda=0.1$.
ReinMax temperatures are tuned over $[1.0,1.1,1.3,1.5]$ and EAT-cubic temperatures over $[0.02,0.05,0.2,0.5]$.
Both experiment families report five independent matched seeds per estimator configuration.

\subsection{Binomial support-size timing}
\label{app:binomial-timing-details}
The benchmark evaluates a batch of $200\times128$ Binomial latents on an NVIDIA L4.
We sweep $n\in\{4,8,16,32,64,128,256,512,1024\}$, corresponding to categorical support size $K=n+1$.
The comparison includes native order-2 and order-3 weights and the released 1D MVE implementation of \citet{hooperetal2026} at $\tau=10^6$.
Figure~\ref{fig:binom-mve-support-scaling} reports milliseconds per forward--backward evaluation.

\subsection{DEF experiment details}
\label{app:DEF-deets}
The DEF experiments use a reduced grid relative to the nonlinear-VAE experiments.
ReinMax temperatures are $\{1,1.3,1.5\}$, EAT-cubic temperatures are $\{0.02,0.05,0.2\}$, and learning rates are $\{3\times10^{-3},10^{-3},3\times10^{-4}\}$.
We use AdamW with batch size 64, train for 500 epochs on 20 Newsgroups and 100 epochs on RCV1, and report five independent matched seeds for each selected configuration.

We use amortized inference because storing document-specific latent variables for RCV1 would be memory-intensive.
The decoder follows the Gaussian-weight, softplus-link Poisson DEF of \citet{def}; its weights are point-estimated, with 30 supertopics and 100 topics.
The encoder is a two-layer neural network mirroring the decoder.
As in the original model, a $\operatorname{Pois}(0.1)$ prior encourages sparsity.

For 20 Newsgroups, we use the master-branch data from \url{https://github.com/akashgit/autoencoding_vi_for_topic_models}, also used by the AVITM paper.
The vocabulary has 1,995 terms, with 11,258 training and 7,487 test documents.
For RCV1, we use the stemmed token files from \url{http://www.ai.mit.edu/projects/jmlr/papers/volume5/lewis04a/a12-token-files/}, retain the 10,000 most frequent terms, pool the documents, and randomly split 794,414 for training and 10,000 for testing.

\subsection{ReinMax and the rate--distortion tradeoff}
\label{app:def-rate-distortion}
On 20 Newsgroups, ReinMax attains better held-out perplexity despite substantially worse train and test ELBO.
We therefore inspect local update directions in selected high-performing runs.

In Figure~\ref{fig:reinmax-bias}, our estimators and EAT-cubic closely track the near-exact reference in the displayed runs.
ReinMax is broadly aligned but places more weight on decreasing distortion, especially in the displayed 20 Newsgroups run, where it moves toward a higher-rate, lower-distortion solution.

Lower distortion improves reconstruction, whereas the ELBO also penalizes KL rate.
ReinMax's local bias therefore acts like a reweighting toward reconstruction.
Its perplexity advantage should not be interpreted as more accurate ELBO-gradient estimation; the displayed direction diagnostic points in the opposite direction.

\begin{figure}[t]
    \centering
    \includegraphics[width=\linewidth]{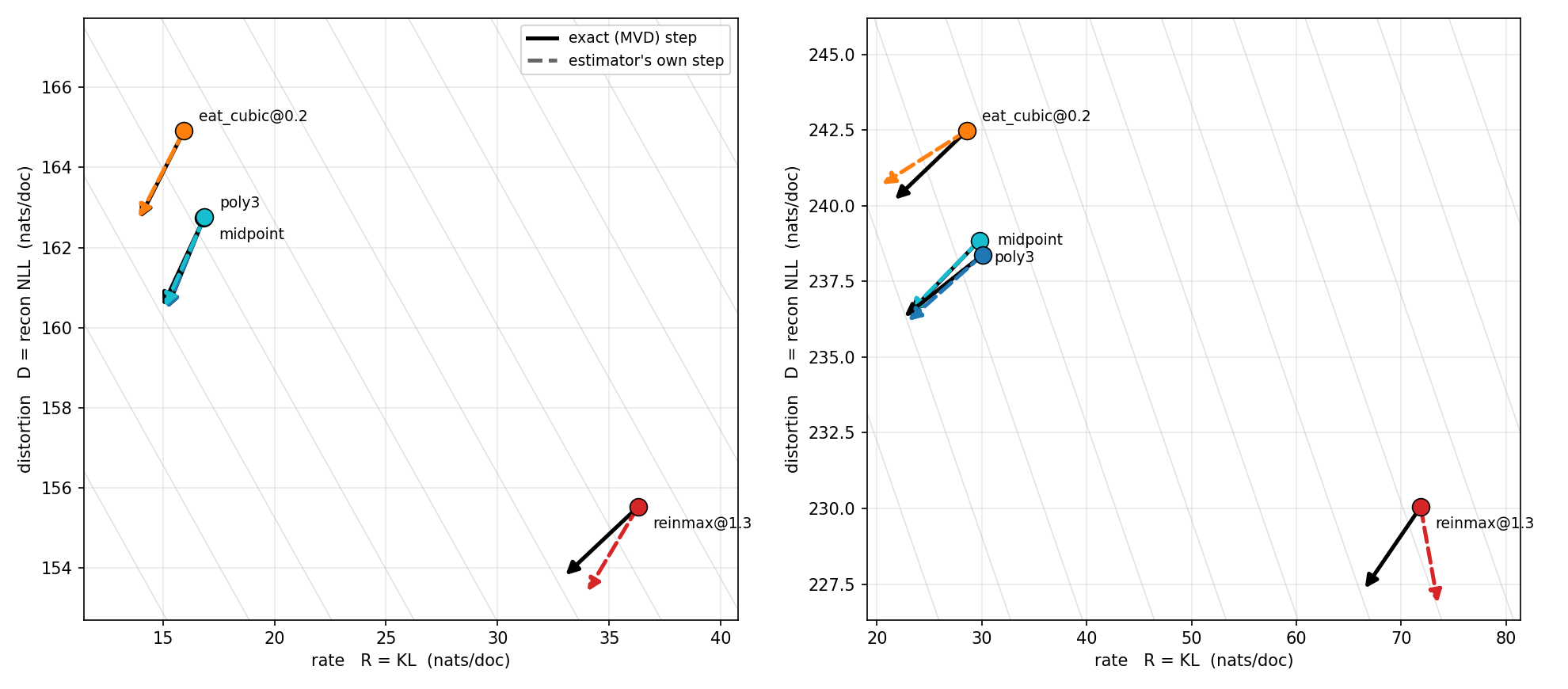}
    \caption{Local rate--distortion update directions in selected high-performing runs. Our estimators closely track the near-exact MVD reference; EAT-cubic is close on 20 Newsgroups but deviates more on RCV1, while ReinMax tilts toward lower distortion and higher rate.}
    \label{fig:reinmax-bias}
\end{figure}

\subsection{Timing protocol}
\label{app:timing-details}
Each value in Table~\ref{tab:speed-results} is the mean $\pm$ standard deviation over five timing repeats after warm-up.
Within a workload, the model, batch size, hardware, and training configuration are fixed across estimators.

\begin{itemize}
    \item \textbf{Linear Poisson VAE:} warm up for two epochs, record cumulative elapsed time at epochs 8 and 24, and divide the difference by 16.
    \item \textbf{Nonlinear Poisson VAE:} warm up for two epochs, record cumulative elapsed time at epochs 3 and 9, and divide the difference by 6.
    \item \textbf{Sparse Poisson DEF:} warm up for two epochs, record cumulative elapsed time at epochs 4 and 12, and divide the difference by 8.
\end{itemize}

We repeat this procedure five times for each estimator--workload--hardware combination.

\end{document}